\documentclass{article}

\usepackage{microtype}
\usepackage{graphicx}
\usepackage{subfigure}
\usepackage{booktabs} 

\usepackage{hyperref}

\usepackage[accepted]{icml2023}

\usepackage{amsmath}
\usepackage{amssymb}
\usepackage{mathtools}
\usepackage{amsthm}

\usepackage[capitalize,noabbrev]{cleveref}

\usepackage{array}
\usepackage{booktabs,makecell,tabularx}
\usepackage{mathtools}
\usepackage{bm}
\usepackage{wrapfig}
\usepackage{enumitem}
\usepackage{graphicx} 
\graphicspath{ {images/} }
\DeclareMathOperator*{\E}{\mathbb{E}}

\usepackage{stfloats} 

\usepackage{titlesec} 
\titleformat{\subsection}[runin]
{\normalfont\normalsize\bfseries}{\thesubsection}{0em}{}
\titlespacing{\subsection}{0pt}{0pt}{*1}

\hypersetup{
    colorlinks = true,
    linkcolor = green,
    citecolor = gray,
    urlcolor = blue,
}

\theoremstyle{plain}
\newtheorem{theorem}{Theorem}[section]
\newtheorem{proposition}[theorem]{Proposition}
\newtheorem{lemma}[theorem]{Lemma}

\theoremstyle{definition}
\newtheorem{definition}[theorem]{Definition}

\newtheorem{property}{Property}
\newtheorem{remark}{Remark}

\usepackage[textsize=tiny]{todonotes}

\icmltitlerunning{Towards Understanding and Improving GFlowNet Training}

\begin{document}

\twocolumn[
\icmltitle{Towards Understanding and Improving GFlowNet Training}

\begin{icmlauthorlist}
    \icmlauthor{Max W. Shen}{genentech,prescient}
    \icmlauthor{Emmanuel Bengio}{recursion}
    \icmlauthor{Ehsan Hajiramezanali}{genentech}
    \icmlauthor{Andreas Loukas}{genentech,prescient}
\end{icmlauthorlist}
\begin{icmlauthorlist}
    \icmlauthor{Kyunghyun Cho}{genentech,prescient,nyu}
    \icmlauthor{Tommaso Biancalani}{genentech}
\end{icmlauthorlist}

\icmlaffiliation{genentech}{Genentech, South San Francisco, USA}
\icmlaffiliation{recursion}{Recursion Pharmaceuticals, Salt Lake City, Utah}
\icmlaffiliation{prescient}{Prescient Design, Genentech, South San Francisco, USA}
\icmlaffiliation{nyu}{Department of Computer Science, New York University, New York, USA}

\icmlcorrespondingauthor{Max W. Shen}{shen.max@gene.com}

\icmlkeywords{generative modeling, gflownet, ICML}

\vskip 0.3in
]


\printAffiliationsAndNotice{} 

\begin{abstract}


Generative flow networks (GFlowNets) are a family of algorithms that learn a generative policy to sample discrete objects $x$ with non-negative reward $R(x)$.
Learning objectives guarantee the GFlowNet samples $x$ from the target distribution $p^*(x) \propto R(x)$ when loss is globally minimized over all states or trajectories, but it is unclear how well they perform with practical limits on training resources.
We introduce an efficient evaluation strategy to compare the learned sampling distribution to the target reward distribution.
As flows can be underdetermined given training data, we clarify the importance of learned flows to generalization and matching $p^*(x)$ in practice.
We investigate how to learn better flows, and
propose
(i) prioritized replay training of high-reward $x$, (ii) relative edge flow policy parametrization, and (iii) a novel guided trajectory balance objective, and show how it can solve a substructure credit assignment problem.
We substantially improve sample efficiency on biochemical design tasks.

\end{abstract}

\vspace{-5mm}
\section{Introduction}
\label{introduction}

A Generative Flow Network (GFlowNet) learns a policy for generating discrete objects $x$, such as graphs, strings, or sets, by iteratively taking actions that add simple building blocks to partial objects (substructures) \cite{bengio2021flow, gfnfoundations}.
GFlowNets view the data-generating Markov decision process (MDP) as a flow network, where ``water", ``unnormalized probability", or non-negative reward $R(x)$ flows through the MDP, from the source node (a ``null'' object), to intermediate nodes (partial objects), to sink nodes (complete objects).
GFlowNets can be seen as an amortized alternative to local exploration methods (e.g., MCMC) that can learn from data to potentially discover
new, distant $x$ with high $R(x)$.
They have been applied to active learning, biological sequence design, and various probabilistic learning tasks \citep{bengio2021flow, gfnbio, discrete-zhang, deleu}.

GFlowNets aim to solve a challenging unnormalized density estimation problem.
Standard learning objectives guarantee that the GFlowNet's learned distribution over $x$ matches the target distribution $p^*(x) \propto R(x)$ when training loss is globally minimized over all states or trajectories, but many
practical domains of interest can have exponentially many $x$ and exponentially many trajectories per $x$, making this infeasible with a practical amount of data and training time.


To gain insight into GFlowNet learning behavior under practical constraints, we design an efficient GFlowNet evaluation scheme that precisely compares the learned sampling distribution to the target reward distribution.
We discover that a primary challenge during GFlowNet training is learning to reduce the probability assigned to low-reward $x$, and that GFlowNets can continue to oversample low-reward $x$ even after 
extensive training time.

When the space of $x$ is large, not all MDP states are seen during training, and 
a GFlowNet's ability to match the target distribution depends on how it generalizes from the flow it learned during training.
GFlowNets were originally designed for the setting where any given $R(x)$ is compatible with multiple flows \citep{bengio2021flow} - that is, \textit{flows are underdetermined}.
Despite this, learned flows have remained understudied, with little to no sense of which learned flows may be more desirable than others.
Our analyses ground a clear notion of optimality for learned flows: a flow is better if it improves generalization to unseen states and helps a GFlowNet match the target distribution better.

We analyze how existing training objectives learn flows, and 
identify a credit assignment problem where the substructures of $x$ most responsible for reward $R(x)$ can be off the sampling trajectory used to reach $x$, and under-credited by popular training objectives.
This is important for compositional reward functions, where the value of $R(x)$ is partially determined by the substructures of $x$.
We propose \textit{guided trajectory balance}, a novel learning objective that is the first solution to this credit assignment problem.

This objective, alongside our proposals of prioritized replay training and relative edge flow parametrization, and even user choice of data-generating MDPs, can change how GFlowNets learn to distribute flow during training. 
In experiments on biochemical design tasks, we demonstrate that these changes in learned flows can significantly impact sample efficiency and convergence to the target distribution, with up to 10$\times$ improvement.
Our work deepens our understanding of the impact of learned flows on GFlowNet behavior, and establishes an fundamental open question: how can we induce GFlowNets to learn more optimal flows, and thereby improve their ability to solve unnormalized density estimation problems.

\section{Preliminaries}

A Generative Flow Network (GFlowNet) learns a generative policy for constructing an object by taking a sequence of actions (see \citep{gfnfoundations} for a more complete description).
This policy acts in a user-defined deterministic Markov decision process (MDP) which must satisfy certain constraints described below. The MDP has a state space $\mathcal{S}$, a set of legal actions $\mathcal{A}_s$ for each state $s$, a deterministic transition function $\mathcal{S} \times \mathcal{A}_s \rightarrow \mathcal{S}$, and reward function $R$. 
GFlowNets view this MDP as a type of directed graph called a flow network where states correspond to nodes and the MDP transition function defines directed edges between nodes.
The \textit{children} of a state are states reachable by outgoing edges, and \textit{parents} are the sources of its incoming edges.
States with no outgoing edges are called \textit{terminal states (sinks)}, and referred to as $x \in \mathcal{X}$.
GFlowNets require the MDP to be defined by the user such that this graph is acyclic, has exactly one state without incoming edges, $s_0$ called the \textit{initial state (source)}, and has $R: \mathcal{X} \rightarrow \mathbb{R}_{\geq 0}$.

A \textit{complete trajectory} is a sequence of states $\tau \triangleq (s_0 \rightarrow s_1 \rightarrow ... \rightarrow s_n)$ starting from source $s_0$ to a sink $s_n$ with $(s_t \rightarrow s_{t+1}) \in \mathcal{A}_{s_t}$ for all $t$. 
We denote $\mathcal{T}$ as the set of all complete trajectories.
A \textit{trajectory flow} is a non-negative function $F: \mathcal{T} \rightarrow \mathbb{R}_{\geq 0}$ describing the unnormalized probability flowing along each complete trajectory $\tau$ from the source to a sink.
For any state $s$, the \textit{state flow} 
$F(s) = \sum_{ \{ \tau \in \mathcal{T}: s\in \tau \} } F(\tau)$ describes the total amount of unnormalized probability flowing through state $s$ and,
for any edge $s \rightarrow s'$, the \textit{edge flow} is 
$F(s \rightarrow s') = \sum_{ \{ \tau \in \mathcal{T}: (s\rightarrow s') \in \tau \} } F(\tau)$.
A trajectory flow $F(\tau)$ is \textit{Markovian} if there exist distributions $P_F({\cdot}|s)$ over the children of every non-terminal state $s$, and a constant $Z$, such that for any complete trajectory $\tau$, we have $P_F(\tau) = F(\tau)/Z$ with $P_F(\tau = (s_0 \rightarrow s_1 \rightarrow ... \rightarrow s_n)) = \prod_{t=1}^{n} P_F(s_t|s_{t-1})$. 

This $P_F(s_{t+1}|s_t)$ is called a \textit{forward policy}, which can sample complete trajectories from $F$. 
We also consider $P_B(s_{t-1}|s_t)$, a \textit{backward policy}.
When $F$ is Markovian, these policies relate to flows by:
$P_F(s'|s) = \frac{F(s \rightarrow s')}{F(s)}$, and
$P_B(s|s') = \frac{F(s \rightarrow s')}{F(s')}$.



\subsection*{Generative Flow Network.}
A GFlowNet is a learning algorithm with parameters $\theta$ comprising a model of a Markovian flow $F_\theta$, and an objective function \citep{gfnfoundations}.
The flow model can be uniquely determined by either:

\vspace{-0.1 in}
\begin{itemize}[leftmargin=*]
    \item edge flows $F_\theta(s\rightarrow s')$, which induces a forward policy 
    $P_F(s'|s)$, or
    \item initial state flow $Z_{\theta} = F_\theta(s_0)$ and forward policy $P_F^\theta(s_{t+1}|s_t)$, or
    \item terminal state flows $F_\theta(x)$ and backward policy $P_B^\theta(s_{t-1}|s_t)$ .
\end{itemize}
\vspace{-0.1 in}

We call the GFlowNet's learned sampling distribution $p_\theta(x)$, which is sampled by starting at $s_0$ and iteratively sampling $P_F^\theta(s_{t+1} | s_t)$ to reach a terminal state $x$.
Learning objectives are designed to match $p_\theta(x)$ to the \textit{target distribution}, $p^*(x) \triangleq R(x) / \sum_{\mathcal{X}} R(x)$.

\subsection*{Trajectory balance objective}
\citep{tb}.
This objective uses a forward policy parameterization by learning $Z_\theta$ and a forward policy $P_F^\theta$, but also learns a backward policy $P_B^\theta$ as a training artifact. 
Trajectory balance is motivated by a detailed balancing condition, which is satisfied for a flow $F$ with state flow function $F(s)$ if $F(s) P_F(s'|s) = F(s') P_B(s|s')$ for all edges $s \rightarrow s'$.
The trajectory balance objective attempts to enforce detailed balancing over a complete trajectory:

\vspace{-0.1 in}
\begin{equation}\label{eq:tb}
    \mathcal{L}_{\texttt{TB}}(\tau) = \Bigg(
        \log 
        \frac{
            Z_{\theta} \prod_{t=1}^n P_F^\theta(s_t|s_{t-1})
        }{
            R(x) \prod_{t=1}^n P_B^\theta(s_{t-1} | s_t)
        }
    \Bigg)^2
\end{equation}

\cite{tb} showed that if a GFlowNet achieves a global minima of this objective over trajectories from a training policy with full support, then it samples from the target distribution. 
When many trajectories can lead to the same $x$, there can be many global optima for the trajectory balance objective - \textit{flows are underdetermined}.
However, for any fixed choice of $P_B$, there is a unique global minima corresponding to a $P_F$ that samples from the target distribution \citep{gfnfoundations}.
\citep{gfnpartial} study an extension of trajectory balance that learns from partial episodes.


\subsection*{Maximum entropy GFlowNets.}
\cite{discrete-zhang} showed that when $P_B$ is fixed as the uniform distribution, 
the unique global minima of the trajectory balance objective is
a Markovian flow $F$ with maximum \textit{flow entropy} $\mathcal{H}[F] \triangleq \mathbb{E}_{\tau \sim P_F} \sum_{t=0}^{n-1} \mathcal{H}[P_F(\cdot | s_t)] $.

\subsection*{Training.}
GFlowNets are trained with stochastic gradient descent to optimize the learning objective on states or trajectories sampled from a training policy, which is usually chosen as a mixture of $P_F^\theta$ and a uniform action policy to encourage exploration during training.
In RL terms, GFlowNet training is off-policy.
Importantly, GFlowNet training is a bootstrapping process: the current policy is used to sample new $x$ at each training round.
As $R(x)$ is defined by the user, it is computed on each new $x$, and this set $\{x, R(x) \}$ is used to update the GFlowNet policy.
We define $X$ as the set of all $x$ seen so far in training. It is initially empty and expands with each training round.

\section{Evaluating GFlowNets and underfitting}

In prior work, GFlowNets have been evaluated in several ways, typically focusing on their ability to discover novel and diverse modes with high reward.
However, their ability to match the target distribution has been empirically studied less thoroughly, especially on real-world data, despite this ability's  central importance in the GFlowNet machinery.

Learning to match the target distribution is non-trivial.
This theoretically occurs when loss is globally minimized over all states or trajectories.
However, in high-dimensional settings where $|\mathcal{X}|$ is exponentially large, visiting each state or trajectory even once is infeasible.

Evaluating GFlowNets can be challenging. 
In many MDPs of interest, there are exponentially many trajectories per $x$, which can make computing $p_\theta(x)$ costly by requiring dynamic programming.
And, $|\mathcal{X}|$ is often exponentially large, making it infeasible to precisely characterize the target distribution.
Several studies evaluate GFlowNets by spearman correlation between $\log p_\theta(x)$ and $R(x)$ on held-out data \citep{gfnpartial, nica2022evaluating} , but this correlation is 1.0 when $\log p_\theta(x) =  c R(x)$ for any constant $c > 0$, even though it is only when $c=1$ that the GFlowNet matches the target distribution.

We design our benchmarks with biological data and rewards, constrain them to have enumerable $\mathcal{X}$, and use GFlowNet samples for evaluation rather than computing $p_\theta(x)$.
This enables an exact, precise, and efficient view of how well GFlowNets match statistics of the target reward distribution.
We study the Anderson-Darling statistic which computes goodness-of-fit between GFlowNet samples of reward and the target reward distribution, and also compare mean GFlowNet reward from samples, $\E_{p_\theta(x)} [R(x)]$, to expected target reward $R(x)^2/Z$.

We remark that this evaluation scheme only compares to the target \textit{reward} distribution.
Nevertheless, this scheme enables us to report that a key challenge in GFlowNet training is underfitting the target distribution, meaning that they sample low rewards with too high probability.
We discuss underfitting here and propose improvements, 
and defer benchmark details and experimental results to \S\ref{section:experiments}.

\subsection*{GFlowNets underfit target distributions during training. }
In our experiments (see \S\ref{section:experiments}), mean GFlowNet sampling reward $\E_{p_\theta(x)} [R(x)]$ initializes significantly below the target mean reward, and very slowly reaches, or never reaches the target mean reward, even over more than tens of thousands of active training rounds. 
At initialization with random parameters, we reason that GFlowNet mean sampling rewards is expected to be low, as $P_F^\theta(s_{t+1}|s_t)$ has high entropy, so $p_\theta(x)$ generally also has high entropy (though it depends on the choice of MDP, see \S \ref{section:appendix_notes}).

This is consistent with previous work, which has reported linear regression slope of 0.58 between $\log p_\theta(x)$ and $\log R(x)$ on a small molecule task \citep{bengio2021flow} after training completed, indicating oversampling of low-reward $x$.
To encourage discovery of high reward $x$, it is common to train on rewards raised to a high exponent, such as 3 to 10.
For instance, 10$\times$ higher binding affinity molecules can be more than 1000$\times$ rarer in $\mathcal{X}$, reducing their relative probability unless reward is taken as binding affinity raised to a power.
This increases reward skew, decreasing target distribution entropy and widens the gap to high-entropy GFlowNet initializations.
In prior work, GFlowNets have failed to sufficiently increase sampled rewards when training on increasingly skewed rewards.

\begin{remark}
\textit{A primary practical challenge during GFlowNet training is reducing their probability of sampling low-reward $x$.}
\end{remark}

Our observations affect the practical use of GFlowNets.
GFlowNet training is conventionally stopped at convergence, or when training resources are depleted, but our experiments show that GFlowNets can converge below the target mean, and fail to reach the target mean even after extensive training time. 
It has not been common practice to monitor the sampling mean reward nor to compare it to the target mean, which is typically unknown.
As a consequence, in practice, many GFlowNets may oversample low-reward $x$, unbeknowst to users.

This motivates us to understand the training behavior of GFlowNets, and explore strategies that may assist them in learning to sample high rewards with higher probability more quickly.
These strategies can help GFlowNets match the target distribution more quickly.

\subsection*{Reward-prioritized replay training (PRT). }
We propose prioritized replay training (PRT) that focuses on high reward data, as a simple strategy for increasing sampled rewards.
At each training round, in addition to training on trajectories sampled from the training policy, we form a replay batch from $X$, all data seen so far, so that $\alpha$\% of the batch is sampled from the top $\beta$ percentile of $R(x)$, and the rest of the batch is sampled from the bottom $1-\beta$ percentile.
To train on the replay batch of $x$, we sample trajectories for them using $P_B$, then train on the trajectories.

Replay training is common in RL, where it can significantly improve sample efficiency \citep{revisiting}.
A similar strategy is error-prioritized experience replay \citep{prioritizedexpreplay}, which preferentially samples replays with high temporal-difference error.
\citep{gfnbio} propose replay training with GFlowNets from a fixed initial dataset.
In contrast, our reward-prioritized replay training is motivated by fixing the observation that GFlowNets can struggle with oversampling low-reward $x$.

\section{Flow Distribution and Generalization}
\label{section:flowdistribution}

We now turn to studying flow distribution, which is how a GFlowNet chooses to flow ``unnormalized probability'' over nodes and edges in the MDP graph. 
In this section, we clarify how flow distribution is important for generalization and efficiently matching the target distribution.
As flows are usually underdetermined, this clarification leads to the question: how can GFlowNets learn better flows?

The first factor in flow distribution is the user's choice of the data-generating MDP. 
For any given data $x$, a user usually enjoys ample choice among different MDPs.
Strings can be generated autoregressively, by starting from an empty string, and iteratively appending characters. In this autoregressive MDP, there is exactly one trajectory ending in any $x$, which means there is only one flow compatible with $R(x)$ over all $x$.
However, if the legal action set is expanded to prepending and appending, then the resulting \textit{prepend/append} MDP has $2^{k-1}$ trajectories ending in a string of length $k$, and there are many flows over the MDP compatible with $R(x)$.
Finally, note that graph-generating MDPs typically have more than two attachment points, and even more trajectories per $x$.

The original motivation for the GFlowNet machinery, and its primary novelty, are in MDPs with many trajectories per $x$ \citep{bengio2021flow}.
It is in this setting that autoregressive or standard RL methods learn biased generative probabilities, while GFlowNets do not.
When trajectories and $x$ are one-to-one, GFlowNets reduce to a known method: discrete-action soft Q-learning.
As such, we focus on the many-to-one setting, where flows are underdetermined given $R(x)$.

\subsection*{The role of generalization during training.}
When $\mathcal{X}$ is high-dimensional and large, it is useful to conceptually divide states, trajectories, and $x$ into those seen in training so far, and those not yet seen. 
The quality of a GFlowNet's policy over $\mathcal{X}$, in particular how closely $p_\theta(x_\texttt{unseen})$ matches $p^*(x_\texttt{unseen})$, critically depends on how well a GFlowNet generalizes from the data seen so far in training.
As GFlowNets train in a bootstrapping manner, improving generalization at any training step improves future training.
Finally, GFlowNets can only ever train on a tiny fraction of $\mathcal{X}$, so their final ability to match the target distribution also depends on generalization.

As $p_\theta(x_\texttt{unseen})$ depends on 
learned forward policy $P_F^\theta$ or equivalently on learned edge flows $F_\theta(s \rightarrow s')$,
we state:

\begin{remark}
\textit{Flow distribution matters for generalization, which matters for matching the target distribution over $\mathcal{X}$.
Moreover, flows are generally underspecified given data $\{x, R(x)\}$.
}
\end{remark}

This remark grounds a notion of the quality of a flow distribution: a flow distribution is better if it helps a GFlowNet generalize better.

\subsection*{Relative edge flow parametrization (SSR). }

Our clarification of the role of generalization in matching the target distribution led us to investigate inductive biases in GFlowNet parametrizations that could impact generalization behavior.
GFlowNets typically use a policy parametrization, where $P_F^\theta(s)$ is a neural net mapping $\mathcal{S} \rightarrow \mathcal{A}$, which we call ``SA''. 
If this policy learns that certain actions are favorable in a state $s$, it can generalize to prefer the same actions in new states similar to $s$, but this can sometimes be the wrong generalization.
This concept is related to optimal transport regularization \citep{gflownetpathreg} to explicitly encourage similar action policies in similar states.

We propose an alternative parametrization, where relative edge flows are parametrized by a neural net 
$f_\theta: \mathcal{S}, \mathcal{S'} \rightarrow \mathbb{R}^+$, which we call ``SSR''. 
We define the probability of transitioning forward from $s, s'$ as 
$  f_\theta(s, s') / \sum_{c \in \texttt{children}(s)} f_\theta(s, c) $.
Intuitively, SSR can ``see'' the child state, unlike SA.
Conceptually, SSR can help the GFlowNet generalize to favor taking different actions at new states than actions seen in training, based on the child state.
We remark that relative edge flow parametrization is less efficient and no more expressive than SA.

SSR is closely related to directly parametrizing state flows $F(s)$ (for example, in the flow matching objective \cite{bengio2021flow}) which may help the GFlowNet generalize from patterns of states with high and low flow to discover new states associated with high reward.
Related work also includes \citep{gfnpartial} which parametrizes state flows to learn from partial episodes, and hypothesizes a benefit to state flow generalization, but their method does not use learned state flows for decision-making.

\section{Compositionality and Credit Assignment}

We have established that flow distribution is important for generalization to unseen $x$, and matching the target distribution over $\mathcal{X}$.
Flows are also generally underspecified given data $\{x, R(x)\}$, and it is an open question of whether GFlowNets prefer to learn certain flow distributions over others. 
In this section, we continue by investigating how GFlowNet learning objectives learn to distribute flow.
In particular, we establish this remark (see fig. \ref{fig:intuition_molecule} for intuition):

\begin{remark}
\textit{When there are many trajectories for each $x$, the manner in which trajectory balance and maximum entropy GFlowNets learn flows
can inadequately credit substructures of $x$ associated with high reward. 
Such substructures can exist when the reward function is compositional.}
\end{remark}

We build up to this remark in pieces: first discussing compositional reward functions, the substructure credit assignment problem, and finally studying TB and MaxEnt objectives.

\subsection*{Compositional reward functions.}

We have established that generalization from seen $R(X)$ to unseen $x$ is important for GFlowNets to match the target distribution when $\mathcal{X}$ is large.
However, it could be that $R(x)$ has no learnable relationship with $x$ - in such a case, generalization is impossible, and no flow distribution is better than any other. 

Thus, the core premise that GFlowNets have an advantage over MCMC for unnormalized density estimation relies on the assumption that generalization of $R(x)$ is possible.
Compositional reward functions, where the value of $R(x)$ depends primarily on properties of subparts or substructures of $x$ and interactions among them \citep{andreas, stanford-compositionality}, are an important example of such learnable $R(x)$.

For example, important properties of small molecules are often caused by the presence of molecular substructures, such as benzene rings.
These examples highlight that when $R(x)$ is compositional, substructures of $x$ can be associated with higher reward.
We call these \textit{important substructures}.

GFlowNets that generate discrete objects iteratively take actions that combine simple building blocks, progressively passing through states corresponding to ``partial objects'' to generate complete objects $x$.
When $R(x)$ is compositional, GFlowNets that assign higher flow to internal MDP states $s$ corresponding to important substructures may generalize better by 
1) increasing $p_\theta(x_\texttt{unseen})$ for downstream $x_\texttt{unseen}$ that contain $s$, increasing the probability of sampling $x_\texttt{unseen}$ associated with higher reward, and
2) enabling the GFlowNet to generalize from $s$ and discover new substructures $s_\texttt{unseen}$ associated with high reward.

\subsection*{The substructure credit assignment problem.}

In reinforcement learning, the credit assignment problem concerns determining how the success of a system's overall performance is due to the various contributions of the system's components \citep{minsky}.
A long trajectory of actions is taken by a learning agent in an environment to receive a final reward, and the agent must learn to assign credit to the actions most responsible for the reward.

We argue that GFlowNets face an analogous challenge - assigning flow, or credit, to important substructures most responsible for $R(x)$ - which we call the \textit{substructure credit assignment problem}.
But in RL, credit assignment is limited to actions taken in a specific trajectory arriving at reward.
However, for GFlowNets, in the typical case where there are exponentially many trajectories ending in any $x$, the substructures most responsible for $R(x)$ are often not on the training trajectory used to reach $x$.

Existing learning objectives inadequately address the substructure credit assignment problem.
Trajectory balance (TB) and MaxEnt objectives take gradients steps to minimize loss on the generative trajectory, from the training policy, used to reach $x$, which usually does not contain important substructures.
Furthermore, TB underspecifies flow: many different flows all globally minimize loss, and there is no discernible inductive bias favoring any particular flow distribution.
MaxEnt fixes $P_B$ as the uniform distribution, achieving a unique credit assignment solution, but one that diffuses credit for $R(x)$ as widely as possible, assigning minimal credit to important substructures.

\subsection*{Prior objectives under-credit important substructures.}

To formalize intuitions, we study the behavior of TB and MaxEnt in a simplified yet representative framework. 
We adopt the tabular model setting used to study algorithms in RL \citep{suttonbarto}, and consider a simple MDP with relevant properties.
Due to space constraints, we provide informal summaries here, and complete statements and proofs in the appendix.

\begin{definition}[Sequence prepend/append MDP setting]
    \label{def:setting_maintext}
    In this MDP, $s_0$ is the empty string, states are strings, and actions are prepending or appending a symbol in an alphabet $\mathcal{A}$. This MDP generates strings of length $n$.
    There is a fixed dataset $\{x, x' \}$ with $R(x) = R(x')$.
    Denote $s^*$ as the important substructure, defined as longest substring shared by $x, x'$, with length $k$.
    Denote $s_k(x)$ the set of $k$-length substrings of $x$.
\end{definition}

\begin{remark}
    In this MDP, there are $2^{n-1}$ trajectories ending in any $x$, and each passes through exactly one string of length $k$, which is either $s^*$, or any other string in the set $\{ s_k(x) \setminus s^* \}$ where $\setminus$ denotes set subtraction.
    We study the flow $F(s^*)$ passing through $s^*$, and the sum of flows passing through $\{ s_k(x) \setminus s^* \}$, 
    using
    $ F( s_k(x) \setminus s^* ) \triangleq
        \sum_{ s' \in \{ s_k(x) \setminus s^* \} } F(s') $
    .
\end{remark}

\begin{proposition}[MaxEnt assigns minimal credit to important substructures]
    In setting $\ref{def:setting_maintext}$, suppose the GFlowNet is trained with the maximum entropy learning objective. 
    Then, at the global minima, 

    \vspace{-0.1 in}
    \begin{equation}
        \E \Bigg[ 
            \frac{F(s^*)}{ F( s_k(x) \setminus s^* ) } 
        \Bigg] = \frac{2}{n-k}
    \end{equation}
    
    with expectation over random positions of $s^*$ in $x, x'$.
\end{proposition}

This proposition states that, for many values of $n, k$ of practical interest, MaxEnt assigns a low minority of flow to $s^*$, and prefers to sample trajectories through non-important substructures with high probability.

To reason about trajectory balance in setting \ref{def:setting_maintext}, we confine it to tabular trajectory flow updates, and ensure equal training steps on $x, x'$.

\begin{definition}[Tabular, fixed-data TB]
    \label{def:tabularfixedtb}
    Suppose that all tabular trajectory flows $F(\tau) = \epsilon$ at initialization, such that $F(x) < R(x)$ for $x, x'$ at initialization.
    We take $m/2$ training steps on each of $x, x'$, sampling a trajectory $\tau_{\rightarrow x}$ ending in that $x$ with probability proportional to current $F(\tau_{\rightarrow x})$, then update $F(\tau_{\rightarrow x}) \leftarrow F(\tau_{\rightarrow x}) + \lambda (R(x) - \epsilon)$.
    
\end{definition}

\begin{proposition}[TB tends to reduce credit to important substructures]
    Suppose setting $\ref{def:setting_maintext}$ and TB training $\ref{def:tabularfixedtb}$. 
    Let $t$ index training steps. 
    At any training step $t$, 
    if 
    $ F_{t-1}(s^*) < (n-k) \E_{s \sim s_k(x) \setminus s^* } [ F_{t-1}(s) ] $
    holds (with expectation under a uniform distribution), then
    the expected trajectory flow update over training trajectories $\tau$ has:

    \vspace{-0.1 in}
    \begin{align*}
        \E_{\tau} \Big[
            F_t(s^*)
        \Big]
        - & F_{t-1}(s^*) 
        <
        \\
        & \E_{\tau} \Big[
            F_t(s_k(x) \setminus s^* )
        \Big]
        - F_{t-1}(s_k(x) \setminus s^* )
        .
    \end{align*}
\end{proposition}

TB learning updates prefer to increase flow through non-important substructures at the expense of $s^*$, except in the rare situation where $F(s^*)$ is significantly larger, by a factor of $(n-k)$, than the average flow through competing non-important substructures.
In \S\ref{thm:tbpolya}, we interpret TB as a Pólya urn process where ``rich gets richer'': trajectories with high flow are sampled more in training, further increasing flow.

In summary, TB and MaxEnt inadequately address the substructure credit assignment problem, and likely do not learn to distribute flow optimally.

\section{Guided Trajectory Balance}
\label{section:guidedtb}

We introduce a novel learning objective called \textit{guided trajectory balance} that enables flexible control over credit assignment and flow distribution.
We propose a specific guide that solves the substructure credit assignment problem.

In general, suppose that for some $x$ and $R(x)$, we would like to assign credit (i.e., flow) for $R(x)$ preferentially along some trajectories.
We express this preference over trajectories ending in $x$ with a \textit{guide distribution} $p(\tau_{\rightarrow x})$.

We emphasize that the guide $p(\tau_{\rightarrow x})$ does not have to be fixed, and does not have to be Markovian with respect to the MDP. 
A particularly powerful approach is defining guides using $X$, the set of all $x$ collected so far during training, which we demonstrate in a following section.
The resulting guide $p(\tau_{\rightarrow x} | X )$ changes with each new active round during training, as $X$ grows.

To formulate a valid and general GFlowNet learning objective, we propose a guided trajectory balance constraint for a trajectory $\tau_{\rightarrow x}$ ending in $x$.
Theorem \ref{thm:2} shows that if a GFlowNet satisfies this constraint at all $x$ and all trajectories, then it samples from the target distribution, which achieves the same asymptotic guarantees as prior GFlowNet learning objectives. 

\begin{theorem}[Guided trajectory balance constraint]
    \label{thm:2}
    Suppose that for any $x$, and any trajectory $\tau_{\rightarrow x} = (s_0 \rightarrow s_1 \rightarrow ... \rightarrow s_n = x)$ ending in $x$, 
    the guided trajectory balance constraint (\ref{eq:gtb_constraint}) holds.
\begin{equation}\label{eq:gtb_constraint}
    Z \prod_{t=1}^n P_F(s_t|s_{t-1})
    =
    p(\tau_{\rightarrow x}) R(x) 
\end{equation}

    Then $P_F$ samples $x$ with probability $R(x)/Z$.
\end{theorem}

\vspace{-0.1 in}
\begin{proof}
    Use $\prod_{t=1}^n P_F(s_t|s_{t-1}) = P_F(\tau_{\rightarrow x})$, then sum over all trajectories ending in $x$ to get $Z P_F(x) = R(x)$.
\end{proof}
\vspace{-0.1 in}


In general, however, the guide distribution may not be Markovian, and
constraint $\ref{eq:gtb_constraint}$ cannot be satisfied everywhere by standard GFlowNets with Markovian flow.
We solve this with a two-phase optimization approach.
Note how guided trajectory balance (\ref{eq:gtb_constraint}) relates to trajectory balance (\ref{eq:tb}), in that $p(\tau_{\rightarrow x})$ plays the role of $P_B$.
We propose to first train $P_B$ to match $p(\tau_{\rightarrow x})$, using: 

\vspace{-0.1 in}
\begin{equation}\label{eq:back-gtb}
    \mathcal{L}_{\texttt{back-GTB}}(\tau_{\rightarrow x}) = \Big(
        \log 
        \frac{
            \prod_{t=1}^n P_B^\theta(s_t|s_{t-1})
        }{
            p(\tau_{\rightarrow x})
        }
    \Big)^2
\end{equation}

which learns a $P_B$ that acts as a Markov approximation to $p(\tau_{\rightarrow x})$. Once converged, we freeze $P_B$, then learn $P_F$ with fixed $P_B$ using trajectory balance, which recovers a proper GFlowNet learning objective: 
by the uniqueness property (\cite{gfnfoundations}, Corollary 1),
for any fixed Markovian $P_B$, 
there is a unique global minima to the trajectory balance loss which corresponds to a Markovian $P_F$ that samples from the target distribution.

\textbf{Training considerations.}
Training trajectories can be sampled from the training policy.
Alternatively, $x$ can be sampled from the training policy, and a training trajectory $\tau_{\rightarrow x}$ can be resampled from the guide.

In practice, a useful $P_F$ can be learned faster by alternating updates to $P_B$ and $P_F$.
$P_F$ may also train on a target that mixes the current $P_B$ and guide $p(\tau_{\rightarrow x})$ with weight $\alpha \in [0, 1]$:

\vspace{-0.1 in}
\begin{align}\label{eq:gtb-mix}
    \mathcal{L}_{\texttt{forward-GTB}}(\tau) &= ( \psi_{\texttt{f}} - \psi_{\texttt{b}} )^2
    \\
    \psi_{\texttt{f}} &\triangleq
        \log Z_\theta + \sum_{t=1}^n \log P_F^\theta(s_t | s_{t-1})
    \\
    \psi_{\texttt{b}} &\triangleq
        \log R(x) 
        + \alpha \log p(\tau_{\rightarrow x})
        \\
        & \quad + (1 - \alpha)
        \Big(
            \sum_{t=1}^n \log P_B^\theta(s_t | s_{t-1})
        \Big) \nonumber
        %
\end{align}
\vspace{-0.1 in}

\subsection*{Substructure guide.}
We propose a particular guide suited for compositional reward functions, used to train \textit{substructure GFlowNets}.
This guide finds important substructures by looking at parts of $x$ associated with high reward over all of $X$ seen so far in training, and guides credit assignment towards these empirically important substructures. 

We say that a state $s$ is \textit{a substructure of} $x$, or $s \in x$, if there is a directed path in the MDP from $s$ to $x$.
A motivating insight is that 
for many compositional objects, $\in$ is efficient to compute (subset function for sets, substring function for strings).
For graphs, $\in$ corresponds to the NP-complete graph isomorphism problem, but this is fast in practice when many node or edge features are distinct.
Furthermore, $\in$ can always be computed by breadth-first search in the MDP DAG, which bounds its time complexity to $O(V+E)$.

Our guide defines $p(\tau_{\rightarrow x} | X) = \prod_{t=0}^{n-1} p(s_{t+1}|s_t, x, X)$,
where state transition probabilities are:

\vspace{-0.1 in}
\begin{equation}
    p(s_{t+1}|s_t, x, X) = \frac
    {
        \phi(s_{t+1} | x , X)
    }
    {
        \sum_{ s' \in \texttt{children}(s_t) } \phi(s' | x , X)
    },
\end{equation}
\vspace{-0.1 in}

with score function $\phi(s|x,X)$ that favors children that are substructures of $x$ with high average reward:

\vspace{-0.1 in}
\begin{equation}
    \label{eq:substructure_score_maintext}
    \phi(s| x , X) = 
    \begin{dcases}
        \E_{ \{x' \in X : s \in x \} } R(x'), & \text{if } s \in x \\
        0 & \text{otherwise}.
    \end{dcases}
\end{equation}
\vspace{-0.1 in}

For ease of presentation, this description ignores corner cases. 
We provide a complete description in \S \ref{section:appendix_subguide}.
In practice, parallelization and an efficiency trick reduce added overhead to negligible amounts (see \S \ref{section:appendix_subguide}).

\subsection*{Substructure GFlowNets credit important shared substructures.}
In setting $\ref{def:setting}$, we showed that MaxEnt and TB objectives inadequately credit the important shared substructure, $s^*$.
In contrast, substructure GFlowNets assign maximal credit to $s^*$ (Proposition \ref{thm:substructure}), demonstrating that substructure-guided trajectory balance helps to solve the substructure credit assignment problem.

\section{Experiments}
\label{section:experiments}

\begin{table*}[b]
\caption{ Converged learned mean vs. target mean (100\% is best), or num. active rounds to match target mean ($\downarrow$) }
\label{table:bigtable}
\begin{center}
\begin{small}
\begin{sc}
\resizebox{.8\textwidth}{!}{\begin{tabular}{l|lllll}
\toprule
Method & Bag & SIX6 & PHO4 & QM9 & sEH \\
\midrule
TB & 100\% (13820 $\pm$ 1740) & 98.0\% $\pm$ 0.2\%  & 78.3\% $\pm$ 0.3\% & 100\% (8405 $\pm$ 185) & 100\% (45840 $\pm$ 500) \\
MaxEnt & 100\% (13775 $\pm$ 2275) & 96.8\% $\pm$ 0.7\% & 77.7\% $\pm$ 0.2\% & 100\% (8735 $\pm$ 775) & 100\% (42670 $\pm$ 190) \\
\midrule
TB + PRT & 100\% (20575 $\pm$ 2355) & 97.3\% $\pm$ 0.2\% & \textbf{80.2\% $\pm$ 0.3\%}  & \textbf{100\% (4735 $\pm$ 765)} & 100\% (19390 $\pm$ 2570) \\
TB + PRT + SSR & 100\% (8880 $\pm$ 1795) & 97.2\% $\pm$ 0.2\% & 79.6\% $\pm$ 0.4\% & 100\% (5630 $\pm$ 650) & 100\% (9920 $\pm$ 4880) \\
Sub + PRT + SSR & \textbf{100\% (7440 $\pm$ 1310)} & \textbf{100\% (4560 $\pm$ 1730)} & 73\% $\pm$ 0.5\% & 98.5\% $\pm$ 1.0\% & \textbf{100\% (5570 $\pm$ 290)}  \\
\bottomrule
\end{tabular}}
\end{sc}
\end{small}
\end{center}
\vskip -0.1in
\end{table*}

To precisely and efficiently evaluate how well GFlowNets match the target reward distribution, we design benchmarks with enumerable $\mathcal{X}$, and use GFlowNet sampled rewards for evaluation.
The Anderson-Darling (AD) statistic, a statistical metric of goodness-of-fit between samples and a target distribution, is strongly correlated with error between mean of the samples and the target mean in our experiments (median $R^2 = 0.87$, \S \ref{section:appendix_experiment}, Fig.~\ref{fig:andersondarling_correlation}), and the AD statistic is near zero (indicating no significant difference) when the mean's error is near zero.
As a result, we report error against the mean for ease of interpretation.
In table \ref{table:bigtable}, we report relative error of the mean at convergence, or number of training rounds required to first reach the target mean.
In figure \ref{fig:trainingcurves}, we depict training curves for the same runs in table \ref{table:bigtable}.
We report mean $\pm$ standard error in table \ref{table:bigtable} and figure \ref{fig:trainingcurves} using three replicates.


In five tasks, we report results from baselines TB and MaxEnt, and compare to the effects of our three proposed training modifications: prioritized replay training (\textbf{PRT}), relative edge flow policy parametrization mapping $\mathcal{S} \times \mathcal{S} \rightarrow \mathbb{R}$ (\textbf{SSR}), and substructure guided trajectory balance (\textbf{Sub}).
We used the ablation series of TB, TB+PRT, TB+PRT+SSR, and Sub+PRT+SSR to compare the individual effects of adding PRT, SSR, and substructure guidance respectively.  
We change the minimal number of hyperparameters and training details necessary to add each proposal (see \S \ref{section:appendix_experiment} for full experimental details).

\subsection*{Bag.} ($|\mathcal{X}| \approx 100$B)
A multiset of size 13, with 7 distinct elements. The base reward is 0.01. A bag has a ``substructure'' if it contains 7 or more repeats of any letter: then it has reward 10 with 75\% chance, and 30 otherwise. In this MDP, actions are adding an element to a bag.

\subsection*{SIX6 (TFBind8).} ($|\mathcal{X}| = 65,536$)
A string of length 8 of nucleotides. The reward is wet-lab measured DNA binding activity to a human transcription factor, SIX6, from \citep{barrera, designbench}.
Following \citep{gfnbio}, we use a reward exponent of 3.
Though an autoregressive MDP is conventionally used for strings, in order to compare TB with substructure guidance which is only meaningful in generative MDPs with more than one trajectory per $x$, we use a sequence prepend/append MDP. We compare to an autoregressive MDP in a following section.

\subsection*{PHO4.} ($|\mathcal{X}| = 1,048,576$)
10 nucleotide string, with wet-lab measured DNA binding activity to yeast transcription factor PHO4 \citep{barrera, designbench}.

\subsection*{QM9.} ($|\mathcal{X}| = 58,765$)
A small molecule graph. 
Reward is from a proxy deep neural network predicting the HOMO-LUMO gap from density functional theory.
We build using 12 building blocks with 2 stems, and generate with 5 blocks per molecule.

\subsection*{sEH.} ($|\mathcal{X}| = 34,012,224$)
A small molecule graph.
Reward is from a proxy model trained to predict binding affinity to soluble epoxide hydrolase from AutoDock.
We build using 18 blocks with 2 stems, and use 6 blocks per molecule.

\subsection*{Results.}

\subsection*{PRT, SSR, and Sub improve GFlowNet training.}
In Table $\ref{table:bigtable}$, we observe that across all 5 experimental settings, the best model includes at least one of our proposals PRT, SSR, and substructure guidance. The best model matches the target mean in the shortest number of rounds - in the case of sEH, 9-fold faster than TB - or achieves the highest mean when all models underfit the target distribution.
We discuss the individual effects of each proposal in following sections.

\subsection*{Baselines and GFlowNets can struggle with underfitting.}

Our benchmark results show that within a reasonable amount of training time, TB converges to policies that undersample the target mean and place too much weight on low rewards in two out of five environments.
This is particularly notable in SIX6 where over 50,000 training rounds TB sampled 800,000 data points (batch size of 16 per round), a relative ratio over $12 \times$ over the state space size $|\mathcal{X}| = 65,536$.
This convergence under the mean despite significantly oversampling the state space size shows the difficulty of finding a global minima over all states and trajectories in practice, and may signal challenges in GFlowNets training in larger state spaces.
Underfitting also varied by environment: all models struggled in PHO4 in particular.

\subsection*{TB and MaxEnt are often similar.}

MaxEnt has a unique global optima from $P_B$ distributing flow uniformly, while TB could in principle learn different flows with lower entropy $P_B$.
Prior work has not yielded much insight into the similarities and differences between MaxEnt and TB. 
In our experiments, we empirically observe that the training curves of MaxEnt and TB are similar, and generally have worse sample efficiency and convergence than our proposals.

\subsection*{Prioritized replay training improves sample efficiency.}
Prioritized replay training (PRT) improves on baselines.
It achieves the highest performance in PHO4 out of all models, and reduces the number of rounds to match the target mean in sEH by 3-fold, from 45840 to 19390.
In environments bag and SIX6, prioritized replay training does not improve the final converged GFlowNet (Table \ref{table:bigtable}), but noticeably improves the sample efficiency of improving mean error in learning curves (compare green vs. blue, Fig. \ref{fig:trainingcurves}).


\subsection*{Effect of relative edge flow parametrization.}

Relative edge flow parametrization (SSR) further improves training by a significant amount (compare red vs. green, Fig.  \ref{fig:trainingcurves}), and in particular noticeably shifts the learning curve in SIX6 by a factor up to 3$\times$.
In Bag and sEH, SSR improves the time to match the target mean by 2-fold.
These experiments show that SSR can meaningfully change the training behavior of GFlowNets, and improve them substantially on some tasks.
Further work may continue to investigate SSR and develop further improved parametrizations.

\subsection*{Effect of substructure guidance.}

Substructure guidance has the biggest effect in SIX6 and sEH, where it enables the only model able to match the target mean in SIX6, and improves sample efficiency in matching the target mean by $2\times$ in sEH (Table $\ref{table:bigtable}$).
Overall, these results confirm that different credit assignment preferences can significantly impact GFlowNet training, convergence, and sample efficiency.

The benefit of our proposed substructure guide can depend on how compositional $R(x)$ is.
Further, failures of the substructure guide to improve training may be due to overfitting.


\begin{figure}[htbp]
\vspace{-0.3in}
    \centering
    \begin{subfigure}{}
    \includegraphics[width=.8\columnwidth,trim={0cm 0cm 0cm 0cm},clip]{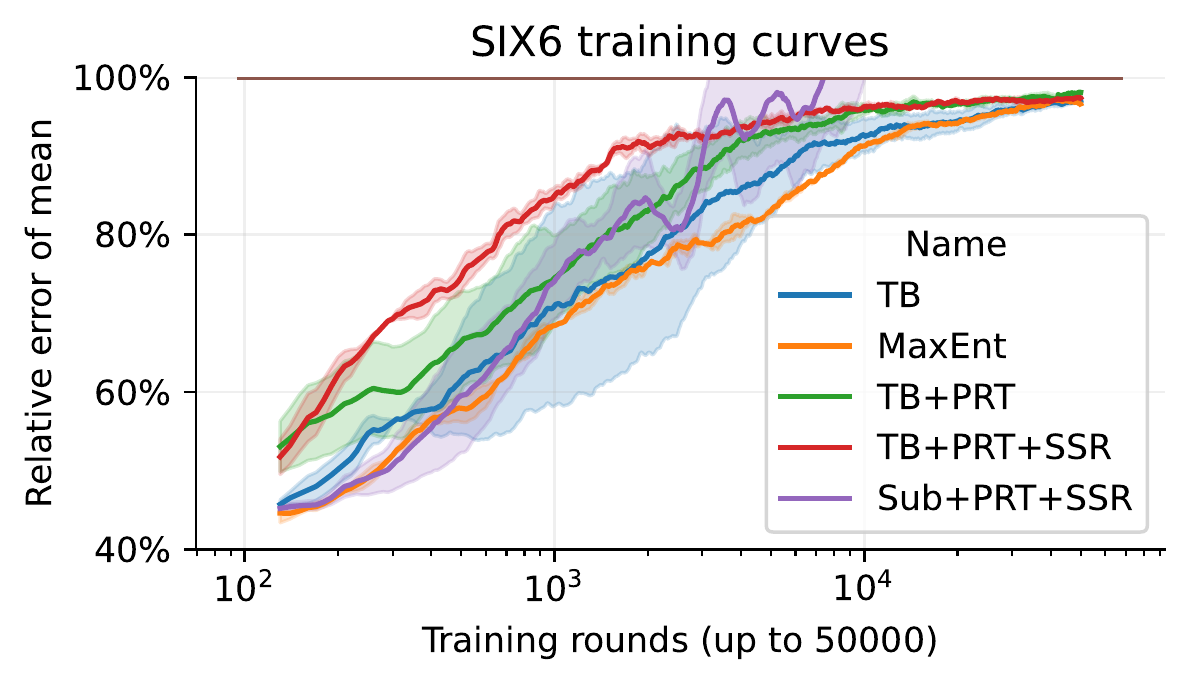}
    \end{subfigure}

    \vspace{-0.1in}
    \begin{subfigure}{}
    \includegraphics[width=.8\columnwidth,trim={0cm 0cm 0cm 0cm},clip]{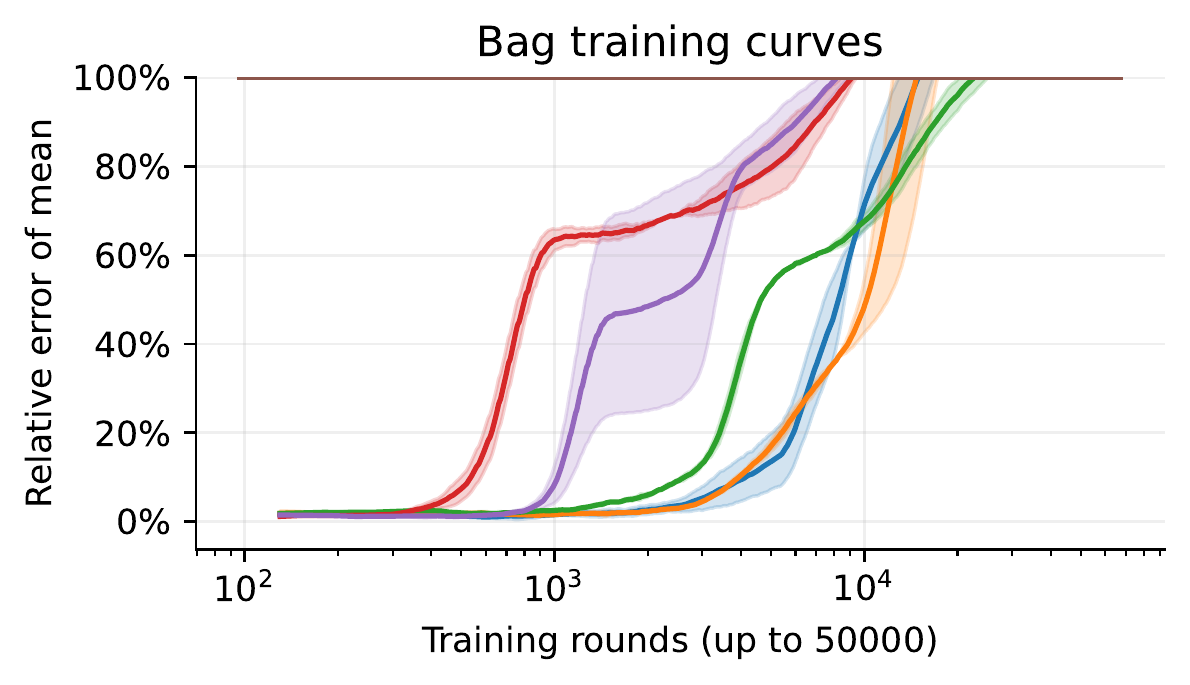}
    \end{subfigure}

    \vspace{-0.1in}
    \begin{subfigure}{}
    \includegraphics[width=.8\columnwidth,trim={0cm 0cm 0cm 0cm},clip]{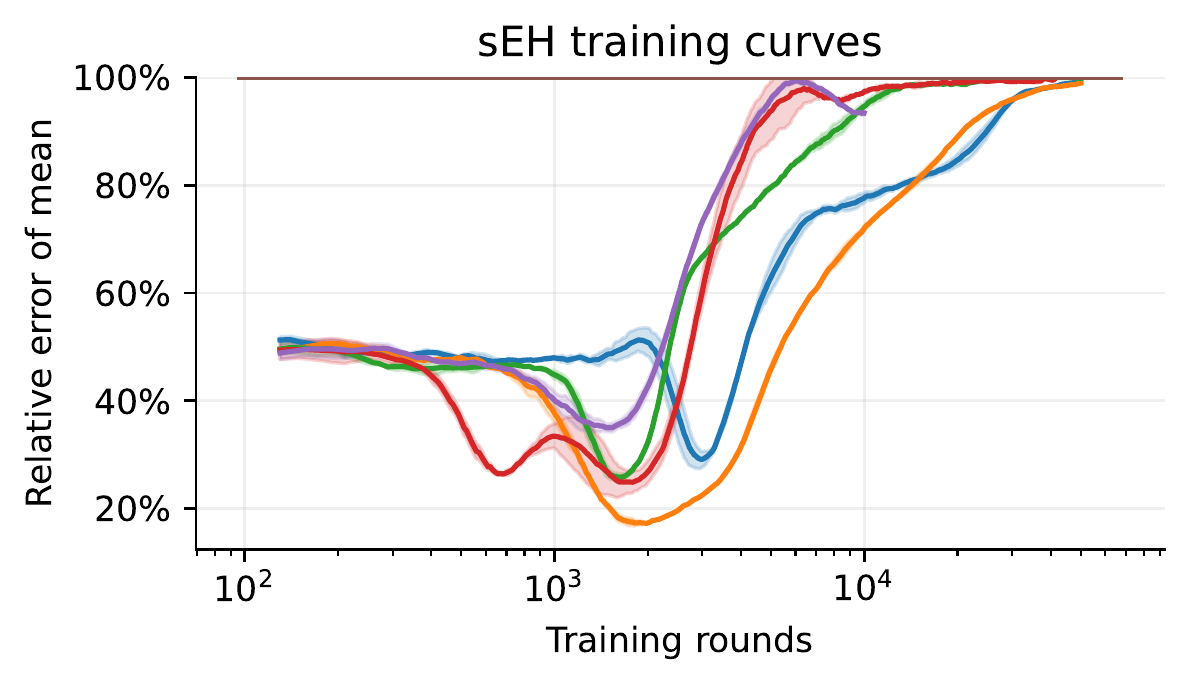}
    \end{subfigure}

    \vspace{-0.1in}
    \begin{subfigure}{}
    \includegraphics[width=.8\columnwidth,trim={0cm 0cm 0cm 0cm},clip]{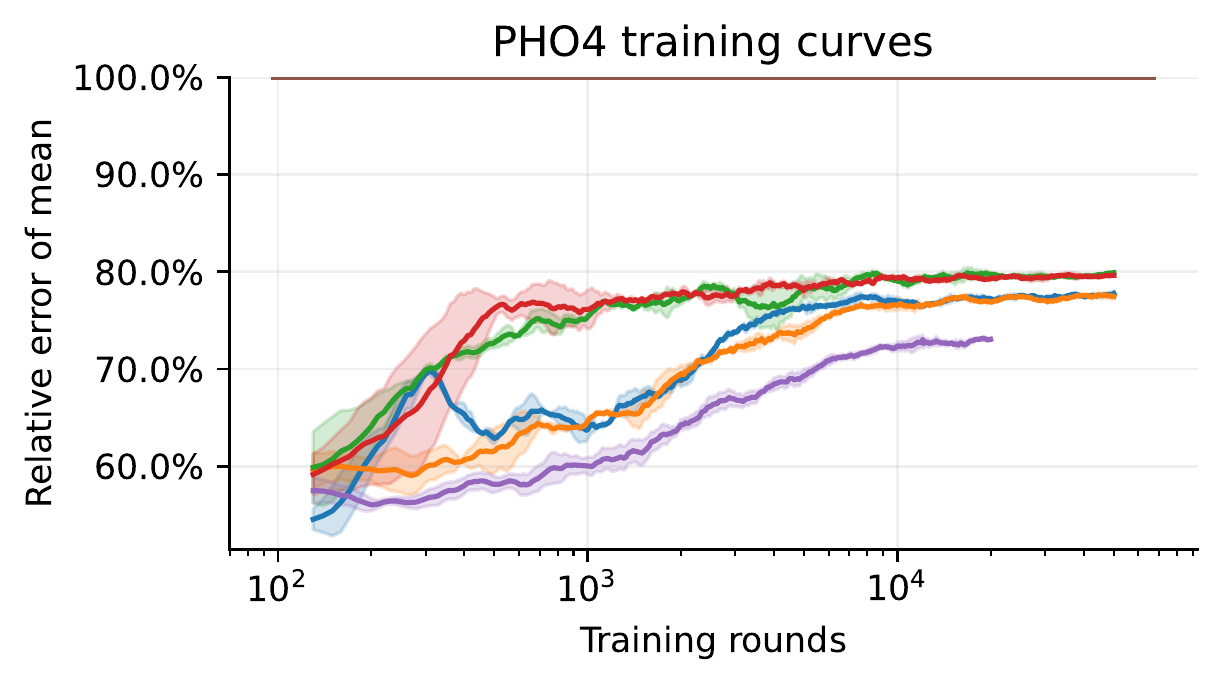}
    \end{subfigure}

    \vspace{-0.1in}
    \begin{subfigure}{}
    \includegraphics[width=.8\columnwidth,trim={0cm 0cm 0cm 0cm},clip]{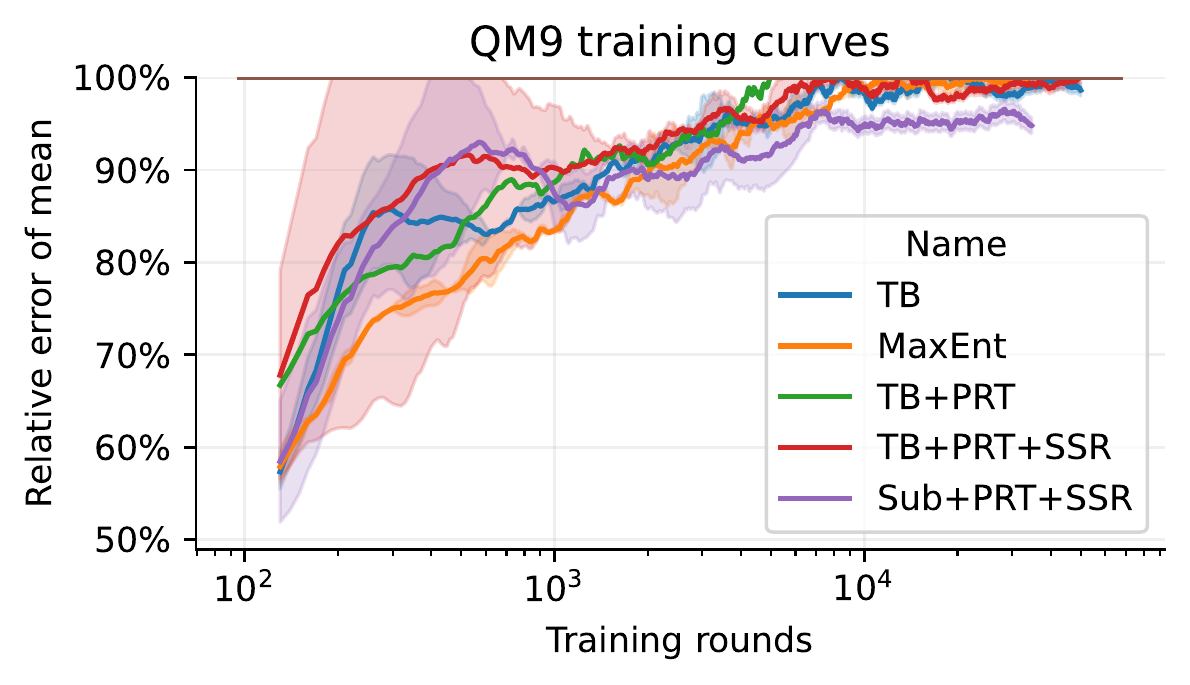}
    \end{subfigure}
    
    \vspace{-0.2in}
    \caption{Training curves. Generally, baselines (orange/blue) increase reward most slowly, and converge later or to lower values, than our proposals (green/red/purple). This effect is seen most clearly in Six6, Bag, and sEH, and less clearly in PHO4 and QM9.}
    \label{fig:trainingcurves}
\end{figure}

\subsection*{Autoregressive MDP.}
\label{section:autoregressive_analysis}

In general, users can have many choices in designing generative MDPs.
TFBind8 is a string-generation environment, for which autoregressive or append-only MDPs have commonly been used for GFlowNet training.
We trained a GFlowNet using an autoregressive MDP, and find final converged mean at 97.3\% after 50,000 training rounds with a TB baseline. 
Note that in an autoregressive MDP, standard TB is equivalent to MaxEnt and substructure guidance has no effect, as there is only one trajectory per $x$.
Adding PRT and PRT+SSR do not impact the final GFlowNet policy, achieving 96.9\% and 96.6\% respectively, but significantly improve sample efficiency, reaching their best policies by rounds 8000 and 3500 respectively, which is a 6-fold and 14-fold improvement (Fig. \ref{fig:autoregressive}). 

Can it be beneficial to use an MDP with more trajectories per $x$? For SIX6, we show the answer is that it can be. 
Using a prepend/append MDP alongside substructure guidance, PRT, and SSR is the only combination in our experiments that managed to match the target mean in SIX6 (Table \ref{table:bigtable}). 
This result shows that favoring autoregressive MDPs may not always be best, as their constraints on building order can limit which $x$ can be built from certain substructures.

\begin{figure}[htbp]
\vspace{-0.1in}
    \centering
    \begin{subfigure}{}
    \includegraphics[width=.8\columnwidth,trim={0cm 0cm 0cm 0cm},clip]{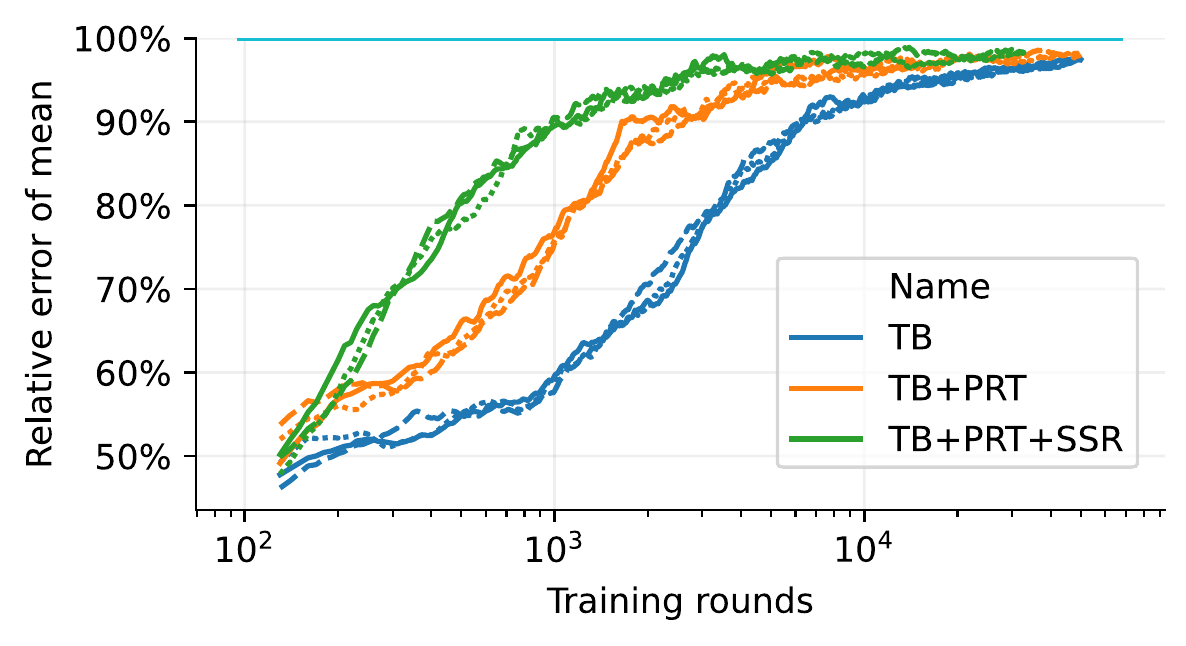}
    \end{subfigure}

    \vspace{-5mm}
    \caption{SIX6, autoregressive MDP. TB baseline (blue) is outperformed with PRT and SSR.}
    \label{fig:autoregressive}
\end{figure}

\vfill\null
\subsection*{Mode discovery and diversity.}

We confirm that improved sample efficiency and sampling higher reward $x$ with higher frequency also improves mode discovery, while retaining diversity.
In sEH, the TB baseline discovered 140 modes by active round 10,000, while Sub+PRT+SSR discovered 867 modes ($6\times$ more).
The Tanimoto diversity ($\uparrow$ is better) among the top 100 reward-ranked samples was 0.55, slightly better than 0.517 for the TB baseline.
This result is consistent with other experiments, where we found that training strategy had no significant impact on sample diversity.

\section{Discussion}

In this work, we have identified challenges with GFlowNet training to learn to match the target distribution in a reasonable number of training steps, and contributed a conceptual understanding of GFlowNet training in terms of generalization, flow distribution, and credit assignment.
We evaluated three proposals for improving GFlowNet training: prioritized replay training, relative edge flow parametrization, and substructure-guided trajectory balance.

We discussed challenges in GFlowNet evaluation, and studied GFlowNet learning behavior on benchmark tasks with biochemical data using the Anderson-Darling statistic between GFlowNet sampled rewards and the target reward distribution, and the difference in sampled mean to target mean. 
We learned that a key challenge in GFlowNet training is learning to not undersample the mean.
However, by evaluating with sampled rewards, we sacrificed exactness for efficiency.
While matching the target mean is better than undersampling it, the target mean can be matched while $p_\theta(x) \neq p^*(x)$ for all $x$.
In future work, more insights into GFlowNet learning behavior may be uncovered by more thoroughly comparing $p_\theta(x)$ to $p^*(x)$ on enumerable MDPs with real-world reward functions.

We have shown that altering flow distributions can significantly change generalization and training behavior, and developed a novel training objective that enables flexible control over credit assignment.
While we have defined one notion of optimality of flow distribution based on generalization and convergence, it remains an open question how best to learn favorable flow distributions.


\clearpage

\bibliographystyle{icml2023}
\bibliography{references}

\newpage
\appendix
\onecolumn

\section{Appendix: Notes}

\subsection{$\quad$ Oversampling low-reward $x$ at initialization}
\label{section:appendix_notes}
Using default PyTorch neural network initializations, we find that $P_F(s_{t+1} | s_t)$ outputting real numbers taken as logits has close to a uniform distribution.
We can understand this with a simple model: suppose $P_F$ is exactly a uniform distribution, all states have exactly the same number of actions $A$, and all trajectories have the same length $n$: then the probability of any trajectory is $1/A^n$.
It is difficult to precisely characterize the entropy of $p_\theta(x)$ at random neural net initialization because it depends on how many trajectories end in each $x$ in the MDP; denote this $N_x$. We can extend the above simple model to find that at initialization, the flow ending in $x$ is proportional to $N_x / A^n$.
Thus $p_\theta(x)$ has highest entropy and is the uniform distribution when $N_x$ is identical for all $x$. However, if $N_x$ varies substantially, then $p_\theta(x)$ will have lower entropy.

However, we note that the common practice of using high reward exponents (from 3 to 10) typically induces a very low entropy reward distribution, which can make it very unlikely in practice that at initialization, a GFlowNet's sampled mean is higher than the target mean.

\subsection{$\quad$ Substructure Guide.}
\label{section:appendix_subguide}

Our substructure guide uses a scoring function to favor children states $s$ that are a member of $x$ in $X$ with high reward.

In equation \ref{eq:substructure_score_maintext}, we presented a simplified scoring function for presentation purposes. The exact scoring function is:

\begin{align}
    \phi(s| x_{\texttt{target}} , X) &= 
        \begin{dcases}
            \E_{ \{x \in X \setminus \{ x_{\texttt{target}} \} : s \in x \} } R(x), & \text{if } s \in x_{\texttt{target}} \\
            0, & \text{otherwise},
        \end{dcases}
\end{align}

In words: the score $\phi$ scores a state $s$, given target $x_{\texttt{target}}$ and $X$, the set of all observed $x$ so far in training.
If the state $s$ is not in ($\in$) the target $x$, then the score is 0 (bottom line).
If the state $s$ is in the target (the state $s$ is a substructure of the target $x$), then the score is the average reward over the subset of $x$ in $X$, where for each $x$ in that subset, $s$ is a substructure of that $x$. This subset excludes the target: we only care about \textit{other} $x$ in $X$ that contains $s$.

When this scoring function is used to build a guide distribution over trajectories ending in $x_{\texttt{target}}$, it prefers trajectories along paths that include substructures/states $s$ that are substructures of high-reward $x$, other than $x_{\texttt{target}}$.

The main difference with the simplified equation \ref{eq:substructure_score_maintext} is that here, we exclude the target $x$ from the collection $X$. This focuses on substructures that are shared with other $x$. 

One corner case is when no child states are members of any $x$ in $X$ except for the target $x$: then the score for all children is zero. In this case, we set $p(s_{t+1} | s_t)$ to the uniform distribution over children $s$ that are members of the target $x$.
We refer readers to the code for a very precise description of the algorithm. 

\subsection*{Implementation details.}
We parallelize guide computation on CPUs, and observe no significant overhead on GPU model training, at the cost of using slightly out-of-date $X$.
We sample guide trajectories efficiently by progressively filtering $X$ as the sampled trajectory lengthens: note that if a state $s$ is not in some $x$, then any descendant of $s$ cannot be in $x$.

\subsection{$\quad$ Additional figures.}
\label{section:additional_figures}

\begin{figure}[htbp]
\vspace{-0.1in}
    \centering
    \begin{subfigure}{}
    \includegraphics[trim={0cm 0cm 0cm 0cm},clip]{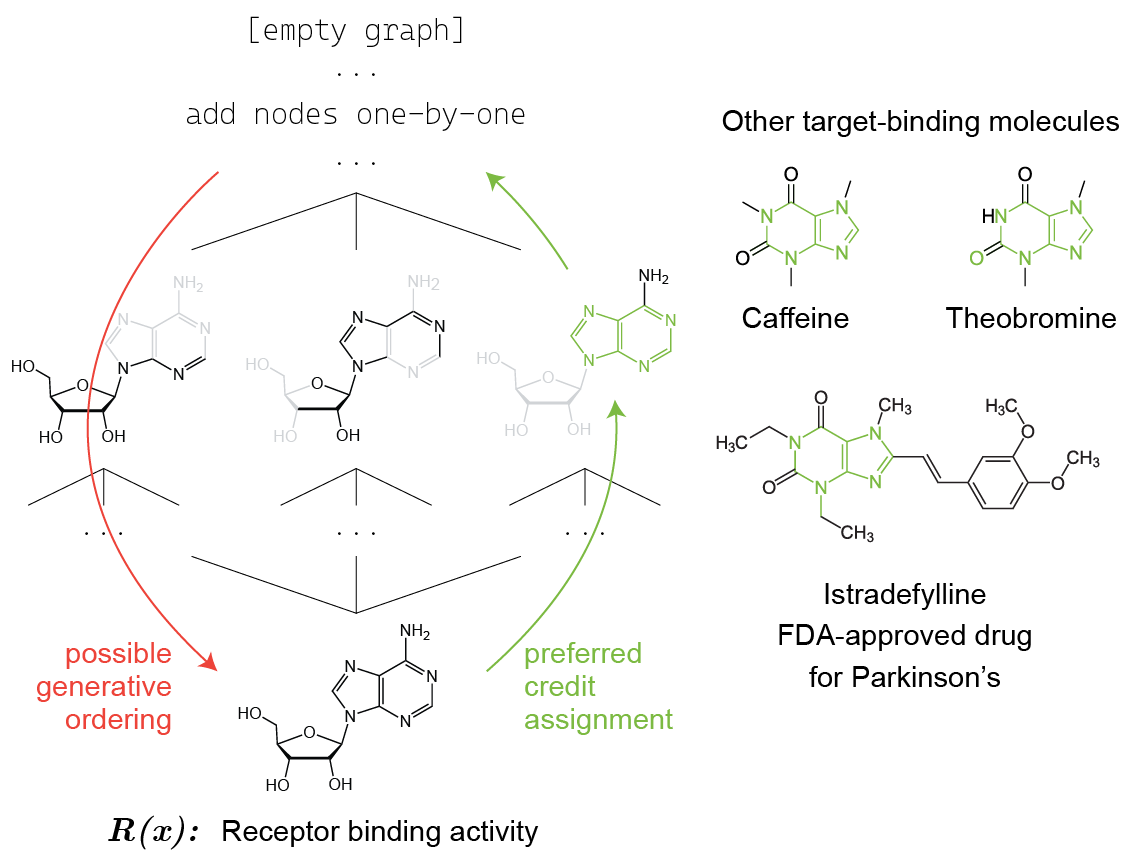}
    \end{subfigure}

    \caption{Intuition figure. (Left) Multiple trajectories lead to the same $x$. An example possible generative ordering (red) disagrees with the preferred credit assignment path (green), which contains a molecular substructure that causes receptor binding activity, a fact that could be learned by a substructure GFlowNet by looking at other molecules (right).}
    \label{fig:intuition_molecule}
\end{figure}

\newpage
\section{Appendix: Experimental Details}
\label{section:appendix_experiment}

Our code is available at \href{https://github.com/maxwshen/gflownet}{https://github.com/maxwshen/gflownet}.

\subsection*{Implementation details.}
We found it useful to clip gradient norms to a maximum of 10.0. We also clipped policy logit predictions to a minimum of -50.0 and a maximum of 50.0. We initialized log $Z_\theta$ to 5.0, which is less than the true $Z = \sum_{x} R(x)$ in every environment. 
In our experiments, every active training round we sampled a batch of 16 $x$ using the current training policy, and updated the model using the sample batch, then optionally performed one replay training update.
For monitoring, every 10 active rounds we sampled 128 $x$ from the GFlowNet policy (not the training policy).
To evaluate the model at a certain round while reducing sample noise, we aggregate over all true policy samples from the last 500 rounds, for an effective sample size of 6400.

There are several design choices for substructure gflownets. One can sample training trajectories from the training policy, or sample $x$ from the training policy and resample a training trajectory using the guide distribution.
We found that the former worked better in practice, though it can increase loss and gradient norm variance, which bounding gradient norm helped with. 
The mixing weight, $\alpha$, in equation \ref{eq:gtb-mix} is another choice. We found that $\alpha = 1$ was generally preferred, which eliminates the $P_B$ term and focuses solely on the guide likelihood term.

For prioritized replay training, we focus on the top 10\% ranked by reward and randomly sample among them to be 50\% of the batch, with the bottom 90\% ranked comprising the remainder of the batch.

\subsection*{Bag.} ($|\mathcal{X}| \approx 100$B)
A multiset of size 13, with 7 distinct elements. The base reward is 0.01. A bag has a ``substructure'' if it contains 7 or more repeats of any letter: then it has reward 10 with 75\% chance, and 30 otherwise. In this MDP, actions are adding an element to a bag.
As this is a constructed setting, we use a small neural net policy with two layers of 16 hidden units. We use an exploration epsilon of 0.10.

\subsection*{SIX6 (TFBind8).} ($|\mathcal{X}| = 65,536$)
A string of length 8 of nucleotides. The reward is wet-lab measured DNA binding activity to a human transcription factor, SIX6, from \citep{barrera, designbench}.
Following \citep{gfnbio}, we use a reward exponent of 3.
Though an autoregressive MDP is conventionally used for strings, in order to compare TB with substructure guidance which is only meaningful in generative MDPs with more than one trajectory per $x$, we use a sequence prepend/append MDP.
We use a neural net with two layers of 128 hidden units and an exploration epsilon of 0.01. 

\subsection*{PHO4.} ($|\mathcal{X}| = 1,048,576$)
10 nucleotide string, with wet-lab measured DNA binding activity to yeast transcription factor PHO4 \citep{barrera, designbench}.
We use a reward exponent of 3, and scale rewards to a maximum of 10.
We use a neural net with three layers of 512 hidden units and an exploration epsilon of 0.01. 

\subsection*{QM9.} ($|\mathcal{X}| = 58,765$)
A small molecule graph. 
Reward is from a proxy deep neural network predicting the HOMO-LUMO gap from density functional theory.
We use a reward exponent of 5, and scale rewards to a maximum of 100.
We build using 12 building blocks with 2 stems, and generate with 5 blocks per molecule. 
As all the blocks have two stems, we treat the MDP as a sequence prepend/append MDP. 
In general, the stem locations on the graph blocks are not symmetric, so reward is not invariant to string reversal.
We use a neural net with two layers of 1024 hidden units and an exploration epsilon of 0.10. 
We measure diversity using 1 - Tanimoto similarity.

\subsection*{sEH.} ($|\mathcal{X}| = 34,012,224$)
A small molecule graph.
Reward is from a proxy model trained to predict binding affinity to soluble epoxide hydrolase from AutoDock. 
Specifically, we trained a gradient boosted regressor on the graph neural network predictions from the proxy model provided by \citep{bengio2021flow} over the entire space of $\mathcal{X}$, with the goal of memorizing the model. Our proxy achieved a mean-squared error of 0.375, and pearson correlation of 0.905.
We use a reward exponent of 6, which is less than 10 used in prior work, and scale rewards to a maximum of 10.
We build using 18 blocks with 2 stems, and use 6 blocks per molecule.
As all the blocks have two stems, we treat the MDP as a sequence prepend/append MDP. 
In general, the stem locations on the graph blocks are not symmetric, so reward is not invariant to string reversal.
We use a neural net with two layers of 1024 hidden units and an exploration epsilon of 0.05. 
We define modes as the top 0.5\% of $\mathcal{X}$ as ranked by $R(x)$.
We measure diversity using 1 - Tanimoto similarity.

\subsection{Anderson-Darling statistic.}

In figure \ref{fig:andersondarling_correlation}, we depict the relationship between the Anderson-Darling statistic between GFlowNet samples and the target distribution, and the relative error between the GFlowNet sampled mean and the target mean.

\begin{figure}[h]
    \centering
    \includegraphics[width=0.5\textwidth]{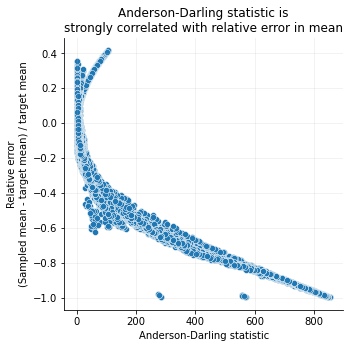}
    \label{fig:andersondarling_correlation}
\end{figure}

\newpage
\section{Appendix: Proofs}
\label{section:appendix_proofs}

In this section, we study the credit attribution behavior of various GFlowNet learning objectives in a simplified yet representative setting.

\subsubsection{Definitions}

\begin{definition}
        \label{def:prependappend}
    \textbf{(Sequence prepend/append MDP). }
    A sequence prepend/append MDP is an MDP where $s_0$ is the empty string, states correspond to strings, and actions are prepending or appending a symbol in some discrete alphabet $\mathcal{A}$ to a string.
    We consider this MDP to only generate strings $x$ of length $n$, i.e., the exit action only occurs at states corresponding to strings of length $n$, and the exit action is the only action available at those states.
\end{definition}

\begin{remark}
    While simple, this is representative of many MDPs that construct compositional objects by combining subparts, such as building graphs by connecting new nodes to existing nodes (often, $>2$ insertion points that increase as the graph enlarges).
\end{remark}

\medskip{}
\begin{definition}
    \label{def:tabulargfn}
    \textbf{(Tabular GFlowNet, with trajectory flow parameterization).}
    A tabular GFlowNet uses a table to store the flow network.
    We consider a trajectory flow parameterization, where the trajectory flow $F: \mathcal{T} \rightarrow \mathbb{R}_{\geq 0}$ is stored as a table.
    In contrast to the function approximation setting, updating any entry in the table does not change any other entry. 
    We assume all trajectory flows $F(\tau)$ are uniformly initialized to a small positive constant $\epsilon$.
    A trajectory flow $F(\tau)$ is updated with $y$ with 
    $F(\tau) \gets F(\tau) + \lambda (y - F(\tau))$ for some step size $0 < \lambda \leq 1$.
\end{definition}

\begin{remark}
    The tabular setting has historically been used to develop, motivate, and study properties of reinforcement learning algorithms \citep{suttonbarto}.
    Note that the trajectory flow parameterization is inefficient in practice, because computing $P_F$ on a single state can require summing over exponentially many trajectories.
    However, it is a useful parameterization for theoretical analysis for various reasons.
\end{remark}

\medskip{}

\medskip{}
\begin{definition}
    \label{def:setting}
    \textbf{(Setting A).}
    This setting comprises a sequence prepend/append MDP, and a fixed dataset $\{x, x'\}$, each of length $n$, with $R(x) = R(x')$.
    Let $s^*$ denote the longest string shared by $x, x'$, with length $k<n$: i.e., there is no string of length $k+1$ shared by both $x, x'$.
    Let $s_k(x)$ denote the set of $k$-length substrings of some $x$.
\end{definition}

\begin{remark}
    In this setting, every trajectory is a sequence of strings of length $(0, 1, 2, ..., n )$, thus every trajectory contains exactly one substring of length $k$ for every $ k \leq n$. 
    As a result, a trajectory ending in $x$ must pass through either $s^*$ (the shared substring) or some substring $s' \in s_k(x)$ (a non-shared substring).
    Our analysis will concern the credit assignment behavior of various learning objectives in how they assign flow to $F(s^*)$ compared to $F(s_k(x) \setminus s^*)$.
\end{remark}

\medskip{}

\subsubsection{Properties}

We now state various properties about setting \ref{def:setting}.

\begin{property}
    \label{property:exponential}
    \textbf{(Exponentially many trajectories for each $x$).}
    In a sequence prepend/append MDP, there are $2^{n-1}$ trajectories that end in a specific string $x$ of length $n$.
\end{property}

\begin{proof}
    To build a string $s$ of length $n$, we can choose an initial zero-based index $i$ for our first character, then we can take $i$ prepend actions and $n-i-1$ append steps in any order.
    There are 
    $\binom{n-1}{i}$
    $n-1$ choose $i$ 
    such sequences of prepend and append actions. Summing over choices of $i$, which can be viewed as summing over a row of Pascal's triangle, there are
    $\sum_{i=0}^{n-1} \binom{n-1}{i} = 2^{n-1}$
    trajectories.
\end{proof}

\begin{remark}
    This property shows how ``easy'' it is to design an MDP with exponentially many trajectories for each $x$.
    Note that an autoregressive, append-only MDP has exactly one trajectory for each $x$.
    The sequence prepend/append MDP has two ``insertion points'', points at which new content can be added, and has exponentially many trajectories for each $x$.
    In general, MDPs that construct objects by combining building blocks will likely have many insertion points.
    MDPs that do not have specific insertion positions, such as sets, also have at least exponentially many trajectories for each $x$.
\end{remark}

\begin{property}
    \label{property:numflowsthroughstate}
    \textbf{(Number of trajectories through a state $s$ to some $x$).}
    In a sequence prepend/append MDP, suppose $s$ of length $k$ is preceded by $a$ characters in $x$ of length $n$.
    There are $\binom{n-k}{a}2^{k-1}$ trajectories passing through $s$ that end at $x$.
\end{property}

\begin{proof}
    There are $2^{k-1}$ trajectories from the root to $s$ (property \ref{property:exponential}).
    There are $\binom{n-k}{a}$ action sequences corresponding to different orderings of prepending and appending, to start from $s$ and end at $x$.
    So in total there are $\binom{n-k}{a} 2^{k-1}$ trajectories passing through $s^*$ that end at $x$.
\end{proof}

\begin{property}
    \label{property:uniform}
    \textbf{Uniform trajectory flow distribution $\iff$ Uniform forward and backward policies.}
    In a sequence prepend/append MDP, the flow has a uniform trajectory flow distribution $F(\tau)$ if and only if $P_F$ and $P_B$ are uniform at every state.
\end{property}

\begin{proof}
    \textit{(Uniform trajectory flow distribution $\implies$ Uniform forward and backward policies).}

    Let the uniform $F(\tau) = \epsilon$.
    Recall that $F(s, s') = \sum_{\tau \in \mathcal{T}: (s, s') \in \tau} F(\tau)$.
    Suppose $s$ has length $k$ and therefore $s'$ has length $k+1$.
    Using property \ref{property:exponential}, there are $2^{k-1}$ trajectories from $s_0$ to $s$, and there are $2^{n-k-1}$ trajectories from $s'$ to any string of length $n$, for a total of $2^{n-2}$ trajectories that pass through the edge $s \rightarrow s'$.
    Using $F(\tau) = \epsilon$, we have for any $s$ with length $k$ and $s'$ with length $k+1$,
    $F(s, s') = 2^{n-2} \epsilon$.
    
    At any state corresponding to a string of length $k$, $P_F$ is only a distribution over strings of length $k+1$, and $P_B$ over $k-1$, because the MDP always adds a single character with every action.
    As $P_F$ and $P_B$ sample with probability proportional to the edge flow, and the edge flow is identical for every edge, $P_F$ and $P_B$ are uniform.
\end{proof}

\begin{proof}
    \textit{(Uniform trajectory flow distribution $\impliedby$ Uniform forward and backward policies).}

    Any trajectory $\tau = (s_0, s_1, ..., s_n)$ is a sequence of strings of length $(0, 1, 2, ..., n)$.
    We consider the forward policy first. Note that at $s_0$, there are $|\mathcal{A}|$ actions/children, corresponding to adding each letter in the alphabet $\mathcal{A}$.
    For any string of length 1 to $n-1$, there are $2|\mathcal{A}|$ actions/children.
    For any string of length $n$, there is 1 action (exiting).
    Using this, we have for any $\tau$, 
    \begin{align}
        p(\tau) &= P_F(s_1 | s_0) \prod_{t=1}^{n-1} P_F(s_{t+1}|s_{t})
        \\
        &= \frac{1}{|\mathcal{A}|} \prod_{t=1}^{n-1} \frac{1}{2|\mathcal{A}|}.
    \end{align}
    
    As this doesn't depend on the states in $\tau$, we have $p(\tau) = p(\tau')$ for any trajectories $\tau, \tau' \in \mathcal{T}$.
    Note that $F(\tau) = Z p(\tau)$.
    
    The backward policy case proceeds similarly, using the observation that each string of length 2 to $n$ has 2 parents.
\end{proof}


\subsubsection{Maximum entropy GFlowNets}

Recall that the maximum entropy learning objective learns $P_F^\theta$ using the trajectory balance constraint:

\begin{equation}
    Z_\theta \prod_{t=0}^{n-1} P_F^\theta(s_{t+1}|s_t) = R(x) \prod_{t=0}^{n-1} P_B(s_t | s_{t+1}) 
\end{equation}

for a trajectory $\tau$ ending in $x$. The maximum entropy objective fixes $P_B(s_{t-1}|s_t)$ as the uniform distribution. 
The trajectory balance constraint is an algebraic manipulation of the simpler detailed balancing constraint

\begin{equation}
    F(s) P_F(s_{t+1}|s_t) = F(s_{t+1}) P_B(s_t | s_{t+1})
\end{equation}

that ``chains'' the detailed balancing constraint over a complete trajectory. 
Other constraints can be derived by extending the chain over partial trajectories.
Note that the trajectory balance constraint with a uniform $P_B$ fully determines $P_F(s_{t+1}|s_t)$ when all $R(x)$ are known.

\begin{theorem}
    \label{thm:maxent}
    \textbf{(Maximum entropy GFlowNets attribute minimal credit to shared substructures).}
    In setting $\ref{def:setting}$, suppose the GFlowNet is trained with the maximum entropy learning objective. 
    Then, at the global minima, 

    \begin{equation}
        \E \Bigg[ \frac{F(s^*)}{F(s_k(x) \setminus s^* )} \Bigg] = \Theta \Bigg( \frac{1}{n-k} \Bigg)
    \end{equation}
    
    over uniformly random positions of $s^*$ in $x, x'$.
\end{theorem}

\begin{remark}
    As all trajectories to $x$ must pass through either the shared substring $s^*$ (with length $k$) or a non-shared substring in $s_k(x) \setminus s^*$, credit assignment for $x, r(x)$ can be evaluated by how much flow, or reward, is assigned to $s^*$ versus $s_k(x) \setminus s^*$.
    This theorem states that a minority of credit is assigned to $s^*$, and it decreases linearly with  $n-k$.
\end{remark}

\begin{proof}
    As a uniform backward policy induces a uniform trajectory distribution over trajectories connecting $x$ to the root (property \ref{property:uniform}), and there are $2^{n-1}$ such trajectories (property \ref{property:exponential}), each trajectory ending in $x$ has trajectory flow $r(x)/2^{n-1}$.
    
    By property \ref{property:numflowsthroughstate}, there are $\binom{n-k}{a} 2^{k-1}$ trajectories passing through $s^*$ that end at $x$.
    Multiplying by the flow over each trajectory, the total flow passing through $s^*$ ending at just $x$ is

    \begin{equation}
        \frac{ \binom{n-k}{a} 2^{k-1} R(x) }{2^{n-1}}.
    \end{equation}

    The average number of trajectories over a uniform distribution on $a$, which can range from $0$ to $n-k$, is:
    
    \begin{equation}
        \frac{1}{n-k+1} \sum_{a=0}^{n-k} \binom{n-k}{a} = \frac{2^{n-k}}{n-k+1} = \E_a \Bigg[ \binom{n-k}{a} \Bigg]
    \end{equation}

    Thus, the expected flow passing through $s^*$ and ending at just $x$, over uniformly random positions of $s^*$ in $x$ is:
    
    \begin{align}
        \E_a \Bigg[ \frac{ \binom{n-k}{a} 2^{k-1} R(x) }{2^{n-1}} \Bigg] =
        \frac{R(x)}{n-k+1}
    \end{align}

    The total expected flow of $s^*$ to both $x, x'$ is $\frac{R(x)+R(x')}{n-k+1}$, while the total expected flow through $s_k(x) \setminus x^*$ is $\frac{(n-k)R(x)}{n-k+1}$.
    When $R(x) = R(x')$, the ratio is $2/(n-k)$, which scales as $\Theta(1/(n-k))$.

\end{proof}

\begin{remark}
    In a set MDP, a similar argument shows that the fraction of flow going through $s^*$ out of all flow going through all subsets of size $k$ scales as $\Theta(1/k!)$.
\end{remark}

\subsubsection{Trajectory balance GFlowNets}

The trajectory balance learning objective is more challenging to analyze because it has many global minima which are reached by a self-modifying learning procedure, where in each step the GFlowNet policy generates samples $x$ which are used to update the policy.



\begin{definition} 
    \textbf{ (Trajectory balance training procedure: Fixed dataset, tabular GFlowNet setting). }
    \label{def:tbtraining}
    
    This definition builds on setting \ref{def:setting}, and the definitions of a tabular GFlowNet (definition \ref{def:tabulargfn}). 
    We suppose that $\epsilon$, the initial value of all trajectory flows, is sufficiently small and $R(x)$ is sufficiently large such that, at initialization,

    \begin{equation}
        F_{\texttt{init}}(x) < R(x).
    \end{equation}
    
    We consider training for $m$ steps, with some learning rate $\lambda$. During each training step, we:
    
    \begin{enumerate}
        \item Sample a trajectory $\tau$ ending in $x$ or $x'$ according to the current GFlowNet policy.
            Under our tabular trajectory flow parameterization, we presume this is done by sampling a trajectory with probability proportional to its flow, among all trajectories ending in either $x, x'$ (see remark \ref{remark:nonmarkov}). 
        \item The GFlowNet is updated according to the trajectory balance learning objective. 
            We update $F(\tau) \gets F(\tau) + \Delta$,
            where we label $\Delta = \lambda (R(x) - F_{\texttt{init}}(\tau))$ as the positive, constant flow update. 
    \end{enumerate}
    
    If $\lambda = 2/m$, then when we have performed exactly $m/2$ training steps on $x$ and $x'$, we have $F(x) = R(x)$ and $F(x') = R(x')$.
\end{definition}


\begin{remark}
    This training procedure is closely related to standard GFlowNet active learning schemas \citep{bengio2021flow}.
    The condition that every trajectory ends in an $x$ in the fixed dataset is also used in the backward trajectory data augmentation training procedure used in \cite{discrete-zhang, gfnbio}.
    This setting is conceptually equivalent to standard active learning when $R(x)$ has the property that learning on any $x'', R(x'')$ for $x'' \notin X$ negligibly impacts the learned flow network. 
    Note that this training procedure has full support over states and trajectories in a tabular GFlowNet where all trajectory flows are initialized to a small positive constant.
\end{remark}

\begin{remark}
    (Non-Markovian trajectory sampling). 
    \label{remark:nonmarkov}
    Sampling a trajectory proportional to its flow when each trajectory has a separate tabular flow value induces a non-Markovian trajectory distribution, which acts as a relaxation of the standard GFlowNet procedure where trajectories are sampled according to the Markovian policy $P_F$ or $P_B$.
    For theoretical analysis, this simplifying assumption ensures that the principle of ``tabular independence'' is obeyed for trajectory sampling - changing the flow of some trajectory $\tau$ does not change the relative likelihood ratio of sampling any other two trajectories $\tau', \tau''$.
\end{remark}

\begin{center}
    \textcolor{gray}{ $\ast$~$\ast$~$\ast$ }
\end{center}

\begin{theorem}
    \label{thm:tb}
    \textbf{(Trajectory balance tends to reduce credit to shared substructures).}
    In setting $\ref{def:setting}$, suppose the GFlowNet is trained with trajectory balance according to the steps in $\ref{def:tbtraining}$. 
    Let $t$ index training steps. 
    At any training step $t$, 
    if

    \begin{equation}
        \label{eq:tbproofcondition}
        \frac{ F_{t-1}(s^*) }{ \E_{s \sim s_k(x) \setminus s^* } [ F_{t-1}(s) ] } < n-k
    \end{equation}
    
    holds (where the expectation is on a uniform distribution over the set), then
    the expected learning update over training trajectories has:
    
    \begin{equation}
        \E_{\tau} \Big[
            F_t(s^*)
        \Big]
        - F_{t-1}(s^*) 
        <
        \E_{\tau} \Big[
            F_t(s_k(x) \setminus s^* )
        \Big]
        - F_{t-1}(s_k(x) \setminus s^* )
        .
    \end{equation}
    
\end{theorem}

\begin{remark}
    This proposition states that, unless $F(s^*)$ dominates $F(s)$ for any $s$ that is not a shared $k$-length substring, each trajectory balance training step tends to increase flow more for non-shared substrings than for $s^*$.
\end{remark}

\begin{proof}
    Recall the following definitions:
    $s^*$ is the longest substring shared between $x, x'$, and has length $k$.
    The set of $k$-length strings in $x$ is denoted by $s_k(x)$, and there are $n-k+1$ such strings.
    
    As any trajectory from $x$ to $s_0$ must pass through exactly one string in $s_k(x)$, and the state flow is defined as the total flow summed over all trajectories passing through that state, the probability of sampling a trajectory that passes through $s^*$ is
    
    \begin{equation}
        p(s^*)
        =
        \frac{F(s^*)}{F(s_k(x))} = 
        \frac{F(s^*)}{ \sum_{s \in s_k(x)} F(s)}
        .
    \end{equation}

    Suppose $F_{t-1}(s^*) < F_{t-1}(s_k(x) \setminus x^*)$. 
    Using the fact that there are $n-k$ items in $s_k(x) \setminus x^*$, this condition is equivalent to equation \ref{eq:tbproofcondition}.
    This means that $p_{t-1}(s^*) < p_{t-1}(s_k(x) \setminus s^*)$.
    
    Note that the expected increase in flow from one training step is:
    
    \begin{align}
        \E_{\tau} \Big[ F_t(s^*) \Big] - F_{t-1}(s^*) 
        &=
        \Delta p_{t-1}(s^*)
        \\
        \E_{\tau} \Big[ F_t(s_k(x) \setminus s^* ) \Big] - F_{t-1}(s_k(x) \setminus s^* )
        &=
        \Delta p_{t-1}(s_k(x) \setminus s^* )
    \end{align}
    
    Using $p_{t-1}(s^*) < p_{t-1}(s_k(x) \setminus s^*)$, we arrive at our proposition.
\end{proof}

\begin{center}
    \textcolor{gray}{ $\ast$~$\ast$~$\ast$ }
\end{center}

When trajectories are sampled according to their flow and not by a Markovian policy,
the ``rich-get-richer'' property can be understood through a Pólya urn model.

\begin{theorem}
    \label{thm:tbpolya}
    \textbf{(Trajectory balance training as a Pólya urn model).}
    In setting $\ref{def:setting}$, suppose the GFlowNet is trained with trajectory balance according to the steps in $\ref{def:tbtraining}$.
    After $m$ training steps, we have: 

    \begin{align}
        p \Bigg(
            \frac{
                F_{\texttt{final}}(s^*) - F_{\texttt{init}}(s^*)
            }{
                F_{\texttt{final}}(s_k(x, x')) - F_{\texttt{init}}(s_k(x, x'))
            }
            > \psi \Bigg)
        \leq
        \\
        1 - \texttt{BetaBinomialCDF} 
            \Bigg(
                \psi m ;
                n = m, 
                \frac{\alpha}{\alpha+\beta} = \frac{e}{\pi \sqrt{n-k}}, 
                \alpha+\beta = \frac{\epsilon 2^{n-1}}{ \Delta}
        \Bigg)
    \end{align}
    
    for $ \frac{e}{\pi \sqrt{n-k}} \leq \psi \leq 1$ and $n-k \geq 3$.
    
\end{theorem}

\begin{remark}
    A consequence of the Pólya urn model is that the fraction of total flow attributed during training that passes through $s^*$ follows a beta-binomial distribution.
    As each training step increases flow by a constant amount, this fraction is equal to the fraction of trajectories sampled through $s^*$ during the ``rich-get-richer'' training process. 
\end{remark}

\begin{proof}
    The Pólya urn model applies as follows: we have one color of ball in the urn for each $k$-length substring in either $x, x'$, where we recall that $s^*$ is the only $k$-length substring shared by both $x, x'$. There are therefore $2(n-k)+1$ colors.
    We denote the set of $k$-length substrings in either $x, x'$ as $s_k(x, x')$.
    At each training step, we sample a trajectory $\tau$ ending in either $x, x'$ which must pass through exactly one $s \in s_k(x, x')$ with probability proportional to $F(s)$; in the Pólya urn, this corresponds to sampling a ball from the urn with color $c$.
    We then increment the flow along $\tau$.
    Due to the tabular trajectory flow parameterization, and tabular trajectory sampling assumption in $\ref{def:tbtraining}$, this incrementing step does not change the flow to any other state in $s_k(x)$.
    In the Pólya urn, this corresponds to adding a ball with color $c$ to the urn. 
    
    The distribution of ball colors sampled from the Pólya urn, after $m$ steps, follows a Dirichlet-multinomial distribution.
    As we only care about $s^*$ compared to $s_k(x, x') \setminus s^*$, we can reduce to two colors, which follows a Beta-binomial distribution.
    We label 'white balls' for $s^*$ and 'black balls' for $s_k(x, x') \setminus s^*$.
    
    At initialization, the number of white balls for $s^*$ is determined by $F(s^*|x, x')$, which depends on its position in $x$ and $x'$ (property \ref{property:numflowsthroughstate}).
    We apply an upper bound on $F(s^*|x)$ to remove this dependence, which will let us treat $x, x'$ identically and simplify our exposition.
    
    Recall that $a$ denotes how many characters precede $s^*$ in $x$.
    We use the upper bound (see lemma \ref{lemma:upperboundpascalrow}):
    
    \begin{equation}
        \max_{0 \leq a \leq n-k} \binom{n-k}{a} \leq \frac{e2^{n-k}}{\pi \sqrt{n-k}}.
    \end{equation}
    
    For some $\hat{x}$ in $\{x, x'\}$, recall that there are $2^{n-1}$ trajectories ending in $\hat{x}$.
    We upper bound the number of trajectories passing through $s^*$ ending in $\hat{x}$ as:
    
    \begin{equation}
        \label{eq:upperboundtrajs}
        \binom{n-k}{a} 2^{k-1} \epsilon
        \leq 
        \frac{e2^{n-k}}{\pi \sqrt{n-k}} 2^{k-1}\epsilon 
    \end{equation}
    
    As both $x, x'$ are the same length, and $s^*$ is the only $k$-length substring in $x$ and $x'$, the upper bound on the proportion of trajectories with $s^*$ are equal for $x$ and $x'$.
    This corresponds to an ``upper-bounded'' Pólya urn where the number of white balls for trajectories with $s^*$ ending in $x$ or $x'$ is proportional to \ref{eq:upperboundtrajs}, and the number of black balls for any other trajectory ending in $x$ or $x'$ is proportional to $2^{n-1}$.
    We continue our analysis using this upper-bounded Pólya urn.
    
    We now solve for the parameters of the beta-binomial distribution that models the urn's draws.
    Let $\alpha$ denote the initial number of white balls for $s^*$. 
    As ``one ball'' is $\Delta$ flow, 
    
    \begin{align}
        \alpha &=
            \frac{e2^{n-k}}{\pi \sqrt{n-k}} \frac{2^{k-1} \epsilon }{ \Delta }
    \end{align}
    
    
    Similarly, let $\alpha + \beta$ denote the initial total number of balls.
    
    \begin{align}
        \alpha + \beta &= (2^{n-k})  \frac{ 2^{k-1} \epsilon }{ \Delta }
    \end{align}
    
    When $n-k = 3$, we have $\frac{\alpha}{ (\alpha + \beta) } = \frac{e}{\pi \sqrt{n-k}} = 0.499555773$, which is the mean of the normalized beta binomial confined between 0 and 1.
    As the beta binomial distribution's mean and variance are largest when $\alpha/(\alpha+\beta) = 0.5$, and decrease as $\alpha/(\alpha+\beta)$ decreases below 0.5, we have that for any threshold $\psi$ greater than or equal to the mean $\alpha / (\alpha + \beta)$, our upper bounded beta binomial has a larger cumulative probability density above $\psi$ than any beta binomial with smaller mean.
    Therefore, our upper-bounded beta binomial satisfies the bounds in the proposition.

\end{proof}

\begin{center}
    \textcolor{gray}{ $\ast$~$\ast$~$\ast$ }
\end{center}

\begin{lemma}
    \label{lemma:upperboundpascalrow}
    \textbf{(Upper bound on the largest element in a row of Pascal's triangle).}
    
    For any natural number $n$,

    \begin{equation}
        \max_{0 \leq a \leq n} \binom{n}{a} \leq \frac{e2^{n}}{\pi \sqrt{n}}.
    \end{equation}
    
\end{lemma}

\begin{proof}
    Note that the series $ \binom{n}{0}, \binom{n}{1}, ..., \binom{n}{n}$ corresponds to a row of Pascal's triangle.

    First, consider when $n$ is even.
    Then $\binom{n}{a}$ is maximized when $a = n/2$.
    We therefore seek an upper bound to $\binom{n}{n/2}$.
    We use Stirling's approximation:
    
    \begin{align}
        \binom{n}{n/2} &\leq \frac{n!}{(\frac{n}{2}!)^2}
        \\
        &\leq \frac{
            e n^{n + \frac{1}{2}} e^{-n}
        }{
            (\sqrt{2\pi} (\frac{n}{2})^{\frac{n+1}{2}} e^{-n/2} )^2
        }
    \end{align}

    Following algebraic manipulations, we get
    
    \begin{align}
        \binom{n}{n/2} &\leq
        \frac{e 2^n}{\pi \sqrt{n}}
    \end{align}
    
    Now, consider when $n$ is odd.
    Then $\binom{n}{a}$ is maximized when $a$ is $n/2$ rounded up or down.
    This corresponds to $n \choose \frac{n+1}{2}$, which we note is equal to $n \choose \frac{n-1}{2}$. 
    By the recurrence relation of Pascal's triangle, we have
    
    \begin{align}
        {n \choose \frac{n+1}{2}} + {n \choose \frac{n-1}{2}}
        &=
        {n+1 \choose \frac{n+1}{2} }
        \\
        {n \choose \frac{n+1}{2}} &= \frac{1}{2} {n+1 \choose \frac{n+1}{2} }.
    \end{align}
    
    Applying the same upper bounding strategy for the even case on ${n+1 \choose \frac{n+1}{2}}$, we get
    
    \begin{align}
        {n \choose \frac{n+1}{2}} &\leq \frac{1}{2} \frac{e 2^{n+1}}{\pi \sqrt{n+1}}
        \\
        &\leq \frac{e 2^{n}}{\pi \sqrt{n+1}}
    \end{align}
    
    As this upper bound is less than the upper bound of $\frac{e 2^n}{\pi \sqrt{n}}$, we can apply that bound whether $n$ is odd or even.
\end{proof}

\subsubsection{Substructure GFlowNets}

\begin{theorem}
    \label{thm:substructure}
    \textbf{(Substructure GFlowNets attribute credit to shared substructures).}
    In setting $\ref{def:setting}$, suppose the GFlowNet is trained as a substructure GFlowNet. 
    Then, at the global minima over learned Markovian policies $P_B, P_F$, 

    \begin{align}
        F(s^*) &= R(x)+R(x'),\\
        F(s_k(x) \setminus s^*) &= 0.
    \end{align}
\end{theorem}

\begin{proof}
    As $s^*$ is the only $k$-length substring in both $x, x'$, all trajectories ending in $x, x'$ sampled by the substructure-aware guide distribution must include $s^*$. When $P_B$ reaches a global minima in matching the guide distribution, it will also have this property, because this property is Markov: for any $s_{k+1}$ of length $k+1$ connected to $s^*$, have $P_B(s^*|s_{k+1}) = 1$. When $P_F$ reaches a global minima, it therefore must assign all flow from $R(x), R(x')$ through $s^*$.
\end{proof}

\begin{remark}
    In practice, it may be preferable to not assign \textit{all} credit for $R(x), R(x')$ to a single substring $s^*$; in some situations, this may correspond to overfitting.
    Other credit assignments are easily achievable by other guide distribution designs, or mixing the substructure GFlowNet guide distribution with a uniform guide distribution.
\end{remark}

\end{document}